\documentclass[10pt,journal,cspaper,compsoc]{IEEEtran}

\ifCLASSOPTIONcompsoc

\else

\fi

\ifCLASSINFOpdf

\else

\fi

\usepackage[hyphens]{url}
\usepackage{graphicx}

\usepackage{epstopdf}
\usepackage{url}
\usepackage{algorithm,algorithmic,cite,amsmath,amssymb,epsfig,enumerate,amsfonts,url,multirow,array,tabularx,booktabs,stfloats,flushend,color}
\usepackage{subfigure,balance}
\usepackage[T1]{fontenc}
\newcommand{\bm}[1]{\mbox{\boldmath{$#1$}}}

\def\x{{\mathbf x}}
\def\w{{\mathbf w}}
\def\y{{\mathbf y}}

\def\e{{\mathbf e}}
\def\K{{\mathbf K}}
\def\M{{\mathbf M}}
\def\W{{\mathbf W}}

\def\Y{{\mathbf Y}}

\def\I{{\mathbf I}}

\def\O{{\mathcal O}}

\def\B{{\mathcal B}}
\def\D{{\mathcal D}}
\def\A{{\mathcal A}}

\def\Z{{\mathcal Z}}

\def\DD{{\mathbf D}}
\def\ZZ{{\mathbf Z}}

\def\tr{{\mbox{tr}}}
\def\wK{{\widetilde{\K}}}
\def\wY{{\widetilde{\Y}}}

\def\wB{{\widetilde{\B}}}

\def\cY{{\mathcal{Y}}}

\def\by{{\bar y}}
\def\bby{{\bar {\mathbf{y}}}}
\def\bbY{{{\mathbf{Y}}}}
\def\coly{{\bar{\bar {\mathbf{y}}}}}
\newtheorem{thm}{Theorem~}

\begin{document}
%
\title{Convex Discriminative Multitask Clustering}
%
%
%
%

\author{~Xiao-Lei~Zhang,~\IEEEmembership{Member,~IEEE}
\IEEEcompsocitemizethanks{\IEEEcompsocthanksitem Xiao-Lei Zhang was with the
Tsinghua National Laboratory for Information Science and Technology, Department of Electronic Engineering, Tsinghua University, Beijing, China, 100084.\protect\\
E-mail: huoshan6@126.com}
\thanks{This work was finished when the author was a visiting scholar with the Department of Computer Science and Engineering, Ohio State University, OH, USA, 43210. This work was supported by the China Postdoctoral Science Foundation funded project under Grant 2012M520278.}
}


\IEEEcompsoctitleabstractindextext{%
\begin{abstract}
Multitask clustering tries to improve the clustering performance of multiple tasks simultaneously by taking their relationship into account.
Most existing multitask clustering algorithms fall into the type of generative clustering, and none are formulated as convex optimization problems.
In this paper, we propose two convex Discriminative Multitask Clustering (DMTC) algorithms to address the problems. Specifically, we first propose a Bayesian DMTC framework. Then, we propose two convex DMTC objectives within the framework. The first one, which can be seen as a technical combination of the convex multitask feature learning and the convex Multiclass Maximum Margin Clustering (M3C), aims to learn a shared feature representation. The second one, which can be seen as a combination of the convex multitask relationship learning and M3C, aims to learn the task relationship. The two objectives are solved in a uniform procedure by the efficient cutting-plane algorithm.
Experimental results on a toy problem and two benchmark datasets demonstrate the effectiveness of the proposed algorithms.
\end{abstract}

}

\maketitle

\IEEEdisplaynotcompsoctitleabstractindextext

\IEEEpeerreviewmaketitle

 \setlength{\arraycolsep}{0.0em}
\section{Introduction}

  With the rapid development of information technology, massive amounts of unlabeled task-specific data are generated every day. Many tasks can be seen as self-contained, yet somewhat similar.
  Because labeling the data manually is time-consuming and expensive, we often resort to \textit{clustering} algorithms for mining the undiscovered knowledge in the data.


In traditional data mining studies, we do clustering to each task independently. However, some tasks have so few data that the data distributions cannot be covered well. Hence, it is natural to think about clustering several unlabeled tasks together for improving the performance on each individual task. However, although some tasks are similar, there are still many tasks mutually unrelated, dissimilar, and even reverse. Simply merging all tasks together for clustering might be harmful. Therefore, \textit{it is urgent to develop a Multitask Clustering (MTC) algorithm that 1) not only is powerful in clustering each individual task 2) but also can mine the task relationships automatically from the data so as to further improve the clustering performance.}
For achieving our goal on MTC, we need to resort to two research areas -- Multitask Learning (MTL) and clustering.


\textbf{Multitask Learning:} MTL \cite{caruana1997multitask}, also known as \textit{learning to learn} \cite{thrun1998learning}, learns multiple (probably) related tasks simultaneously for improving the generalization performance on each task. It can be reviewed in three respects. They are 1) ``what to learn'', 2) ``when to learn'', and 3)``how to learn''  \cite{pan2010survey}.


``What to learn'' asks what knowledge is shared across tasks \cite{pan2010survey}. In this respect, the MTL techniques can be categorized to two classes:
 The first class is to share common feature or kernel representations, such as sharing the hidden units of neural networks \cite{caruana1997multitask,bakker2003task,bengio2011deep}, sharing a common representation within the regularization framework \cite{ando2005framework, evgeniou2007multi,  argyriou2008convex, liu2009multi,zhang2010probabilistic,chen2009convex,chen2012convex}, etc.
The second class is to share common model parameters, such as placing a common prior across tasks within the hierarchical Bayesian framework \cite{lawrence2004learning,xue2007multi,liu2009semisupervised}, learning the differences of the task-specific models in Frobenius norms under the regularization framework \cite{evgeniou2004regularized,evgeniou2006learning, zhang2010convex}, etc.

``When to learn'' asks in which situation the tasks can share. Specifically, many MTL algorithms assume that the tasks are mutually related which is an ideal situation. In practice, there might be some outlier tasks or tasks with negative correlation. Learning with these tasks results in \textit{negative transfer} or worsened performance. Hence, how to discover the task relationship is another key issue that is becoming more and more attractive \cite{bakker2003task,jacob2008clustered,zhou2011clustered,romera2012exploiting,zhang2010convex}. One method is to group tasks into several clusters where the tasks in different groups are regarded as unrelated \cite{bakker2003task,jacob2008clustered,zhou2011clustered,romera2012exploiting}. Another method is to learn the inter-task covariance matrix of the multivariate Gaussian prior \cite{zhang2010convex}.

``How to learn'' asks how the optimization problem can reach a good solution (i.e. performance) in a reasonable time when the first two respects are specified. In respect of effectiveness, among the aforementioned MTL methods, how to construct convex optimization objectives is a common thought in MTL since the global optimum solutions can be achieved and the optimization can be simplified.
Until present, several convex MTL algorithms have been developed, and better performance was reported \cite{argyriou2008convex,jacob2008clustered,chen2009convex,chen2012convex,zhang2010convex}. In respect of efficiency, the alternating optimization method that optimizes in turn one parameter with others fixed is a common efficient method.

Summarizing the aforementioned, in the new MTC design, we take the convexity and the task relationship mining as two important considerations.

\textbf{Clustering:}
Clustering is the process of partitioning a set of data observations into multiple clusters so that the observations within a cluster are similar, and the observations in different clusters are very dissimilar \cite{han2011data}.
 Since the early works on $k$-means, many clustering algorithms have been developed, such as kernel $k$-means, spectral clustering \cite{shi2002normalized,ng2001spectral}, hierarchical clustering, probabilistic-based clustering, metric clustering, clustering nonnumerical data, clustering high dimensional data, clustering graph data, etc.

Like supervised classification, clustering algorithms can be classified to two classes -- \textit{generative clustering} and \textit{discriminative clustering}. The generative clustering algorithms model $p(\x,y;\theta)$ where $\x$ and $y$ denotes the input and output of the learning system respectively and $\theta$ is the parameter. The discriminative clustering algorithms only focus on modeling $p(y|\x;\theta)$. Many traditional clustering algorithms fall into the class of the generative clustering, such as $k$-means, Gaussian mixture model, restricted Boltzman machine, etc. However, when we only care about the predicted labels but not the distribution of the observations, the generative clustering methods seem solving a more general problem than what we want. Moreover, if we make a wrong model assumption on the underlying data distribution, we may get a rather weak clustering result. This phenomenon has been observed in both the supervised classification \cite{jordan2002discriminative} and the clustering \cite{bach2007diffrac}. Due to the above problems, many discriminative clustering methods have been developed \cite{ng2001spectral,ye2007discriminative,bach2007diffrac,xu2005maximum,
xu2005unsupervised,zhang2009maximum,wang2010mmc,li2009tighter,zhang2012linearithmic,gomes2010discriminative,wang2012discriminative}, such as spectral clustering \cite{ng2001spectral}, Maximum Margin Clustering (MMC) \cite{xu2005maximum,xu2005unsupervised,zhang2009maximum,wang2010mmc,li2009tighter,zhang2012linearithmic}, regularized information maximization \cite{gomes2010discriminative}, etc.

Summarizing the aforementioned, in the new MTC design, we should try to construct a discriminative MTC clustering algorithm but not a generative one.

\textbf{Multitask Clustering:}
Although the supervised MTL has been studied extensively in the aforementioned respects, the unsupervised MTL, i.e. MTC \cite{Teh05sharingclusters}, seems far from explored yet. Only very recently, it received more and more attention \cite{Teh05sharingclusters,dai2008self,gu2009learning,zhang2010multitask,zhang2011multitask,gu2011learning,
thach2011compression,xie2012multi,huy2012feature,jiang2012transfer,zhang2012multi}. 1) In respect of ``what to learn'', in \cite{Teh05sharingclusters}, Teh \textit{et al.} proposed to discover the clusters that can be shared via the hierarchical Dirichlet process. In \cite{jordan2012revisiting}, Kulis and Jordan first revisited a regularized $k$-means algorithm in the view of the Dirichlet process and then extended it to MTC by sharing the clusters of the observations across the tasks. In \cite{dai2008self}, Dai \textit{et al.} extended the information theoretic co-clustering algorithm to MTC by making the tasks share the same feature attribute cluster, where they studied MTC in the \textit{transfer learning} scenario, a special case of MTL that focuses on the performance of one target task. In \cite{gu2009learning,gu2011learning,zhang2010multitask,zhang2011multitask,zhang2012multi,thach2011compression,huy2012feature}, the authors tried to learn a shared feature or kernel representations in different distance metrics, such as Bregman distance. 2) In respect of ``when to learn'', in \cite{zhang2010multitask,zhang2011multitask}, Zhang and Zhang proposed the pairwise task regularization and centralized task regularization methods for discovering the task relationship. 3) However, in respect of ``how to learn'', none of the MTC algorithms can hold the convexity.

Moreover, most of the MTC algorithms belong to the class of the generative clustering. To our best knowledge, the discriminative MTC seems lack of full study. Only in \cite{gu2011learning,jiang2012transfer}, the authors proposed the spectral clustering based MTCs.

\textbf{Contributions:}
In this paper, we propose a new Bayesian Discriminative MTC (DMTC) framework. We implement two DMTC objectives by specifying the framework with four assumptions. The objectives are formulated as difficult Mixed Integer Programming (MIP) problems. We relaxed the MIP problems to two convex optimization problems. The first one, named convex Discriminative Multitask Feature Clustering (DMTFC), can be seen as a technical combination of the convex supervised Multitask Feature Learning (MTFL) \cite{argyriou2008convex} and the Support Vector Regression based Multiclass MMC (SVR-M3C) \cite{zhang2012linearithmic}. The second one, named convex Discriminative Multitask Relationship Clustering (DMTRC), can be seen as a technical combination of the convex Multitask Relationship Learning (MTRL) \cite{zhang2010convex} and SVR-M3C.
These combinations are quite natural and yield the following advantages:
\begin{enumerate}[1)~]
\item In respect of ``what to learn'', DMTFC can learn a shared feature representation between tasks. DMTRC can minimize the model differences of the related tasks. Both algorithms, as discriminative clustering algorithms, try to find the optimal label pattern directly. Both of them work in Frobenius norms under the regularization framework.
    \item In respect of ``when to learn'', DMTRC can learn the task relationship automatically from the data by learning the inter-task covariance matrix.
    \item In respect of ``how to learn'', both algorithms are generated from the Bayesian framework. Both of them are formulated as convex optimization problems, and are solved in a uniform optimization procedure. A number of efficient SVM techniques are available for the problems. In this paper, we employ the cutting-plane algorithm \cite{kelley1960cutting,xu2009extended,yang2011efficient} that has achieved a great success in SVM to solve the DMTCs efficiently.
\end{enumerate}
 Experimental comparison with 7 single task clustering algorithms and 3 state-of-the-art MTCs on the pendigits toy dataset, the multi-domain newsgroups dataset, and the multi-domain sentiment dataset demonstrates the effectiveness of the proposed DMTCs.

The remainder of the paper is organized as follows. In Section \ref{sec:related_work}, we briefly review two related techniques -- the convex MTL and the convex MMC. In Section \ref{sec:prob}, we propose a Baysian framework for DMTC. In Sections \ref{sec:DMTFC} and \ref{sec:DMTRC}, we present the covex DMTFC and DMTRC objectives respectively. In Section \ref{sec:solution}, we solve DMTFC and DMTRC within a uniform optimization procedure. In Section \ref{sec:kernel}, we extend DMTC to nonlinear kernels. In Section \ref{sec:complexity}, we analyze the complexity theoretically. In Section \ref{sec:experiments}, we show the effectiveness of DMTC empirically. Finally, in Section \ref{sec:conclusion}, we conclude this paper and present some future work.

We first introduce some notations here. Bold small letters, e.g., $\w$ and $\bm\alpha$, indicate column vectors. Bold capital letters, e.g., $\W$, $\K$, indicate matrices. Letters in calligraphic bold fonts, e.g., $\mathcal{A}$, $\mathcal{B}$, and $\mathbb{R}$, indicate sets, where $\mathbb{R}^d$ denotes a $d$-dimensional real space. $\mathbf{0}_m$ ($\mathbf{1}_m$) is a vector with all $m$ entries being 1 (0). $\mathbf{I}_{d}$ is a $d\times d$ identity matrix.  The operator $^T$ denotes the transpose. The  $\left< \x,\y\right>$ defines the inner product of $\x$ and $\y$. The operator $\|\cdot\|^m$ denotes the $m$-norm, where $m$ is a constant. The operator ``$\mbox{tr}(\cdot)$'' denotes the trace of matrix. The abbreviation ``s.t.'' is short for ``subject to''. $h(\bm\alpha;\bm\beta)$ denotes a function $h$ with parameters $\bm\alpha$ and $\bm\beta$. The symbol $\{\W_c\}_{c=1}^{C}$ is short for the set $\{\W_1,\ldots,\W_C \}$. Without confusion, we may further write $\{\W_c\}_{c=1}^{C}$ as $\{\W_c\}_{c}$ in equations for simplicity.

\section{Related Work}\label{sec:related_work}

\textbf{Convex Multitask Learning: }
We introduce some related convex MTL \cite{argyriou2008convex,jacob2008clustered,chen2009convex,chen2012convex,zhang2010convex,zhou2011clustered} as follows.

 In \cite{jacob2008clustered}, Jacob and Bach proposed to learn the task relationship by clustering the similar tasks into the same group. Because the embedded clustering problem is non-convex, they relaxed the problem to a convex one. In \cite{zhou2011clustered}, Zhou \textit{et al.} proved that the alternating structure optimization (ASO)\cite{ando2005framework} and the clustered MTL (CMTL)\cite{jacob2008clustered} are equivalent except that ASO operates on the feature dimension of the multitask model but CMTL operates on the task dimension of the model. Observing the equivalence, in \cite{chen2009convex,chen2012convex}, Chen \textit{et al.} proposed a convex ASO that learns a shared feature subspace.

In \cite{argyriou2008convex}, Argyriou \textit{et al.} proposed to minimize the empirical risk of all tasks with a Frobenius norm penalty on the differences of the task-specific models, which is a non-convex optimization problem. Then, they proved that the problem is equivalent to a convex optimization problem -- Multitask Feature Learning (MTFL). In \cite{zhang2010convex}\footnote{Best Paper Award of \textbf{UAI-2010}}, Zhang and Yeung first tried to learn the task covariance matrix of the multivariate Gaussian prior in the regularization framework. Because the concave function with respect to the covariance matrix variable makes the objective non-convex, they further replaced the concave function by two convex constraints, which results in a convex optimization problem, named MTRL.

We found that the relationship between MTFL and MTRL are similar with that between ASO and CMTL. Both MTFL and MTRL can be explained together in the Bayesian framework, which contributes to our motivation on the Bayesian DMTC framework.

 But, to prevent misleading, here, we have to emphasize that convex formulations do not mean absolutely better performance over non-convex ones. How to find good local minima in the non-convex formulations seems not a well explored field in MTL, but is emerging in the study of the regularization frameworks, such as \cite{gong2013general} and the references therein.


\textbf{Convex Maximum Margin Clustering: }Among the numbers of discriminative clustering algorithms, MMC \cite{xu2005maximum,xu2005unsupervised,zhang2009maximum,wang2010mmc,li2009tighter,zhang2012linearithmic},
which is an unsupervised extension of Support Vector Machine (SVM), has received much attention since year 2005.
The key idea of MMC is to find not only the maximum margin hyperplane in the feature space but also the optimal label pattern, such that if an SVM trained on the optimal label pattern, the optimal label pattern will yield the largest margin among all possible label patterns $\{\y|\y=\{y_j\}_{j=1}^{n},\forall y_j \}$, where $n$ is the number of observations and $y_j$ denotes the possible class of the $j$-th observation.
The main difficulty of MMC lies in that it is originally formulated as a difficult Mixed-Integer Programming (MIP) problem \cite{xu2005maximum} due to the integer vector variable $\y$ in the objective of MMC.



  To overcome MIP, researchers either relaxed the objective as convex optimization problems \cite{xu2005maximum,xu2005unsupervised,li2009tighter,zhang2012linearithmic} or reformulated it to non-convex ones \cite{zhang2009maximum,wang2010mmc}. Because the convex relaxation methods achieve better clustering results than non-convex ones in general, we pay particular attention to this kind.

Originally, in \cite{xu2005maximum}, Xu \textit{et al.} proposed to reformulate MMC as a convex semi-definite programming problem by relaxing $\M=\y\y^T$ to a continuous matrix. In \cite{xu2005unsupervised}, they further extended the binary-class MMC to the multiclass scenario which has a time complexity as high as $\O(n^{6.5})$.
Recently, in \cite{zhang2012linearithmic}, Zhang and Wu proposed to construct a convex hull \cite{boyd2004convex} on $\{ \y\}$, and further extended the binary-class algorithm to the multiclass problem, i.e. SVR-M3C, which can be solved in an alternating method in time $\O(n\log n)$.

We found that SVR-M3C and MTFL/MTRL can be combined quite naturally within the proposed DMTC framework, and a number of popular SVM techniques are available for solving the problem efficiently. Therefore, MMC contributes to the implementation of the proposed DMTC framework.

\textbf{Cluster Ensemble: }The most similar work with MTC in machine learning and data mining is \textit{cluster ensemble} \cite{strehl2003cluster,topchy2005clustering,fred2005combining,wang2009bayesian,vega2011survey,gomes2011crowdclustering,yi2012robust,yi2012crowdclustering,wang2011nonparametric}. The cluster ensemble aims to combine multiple clusterings with a so-called \textit{consensus function} for enhancing the stability and accuracy of the base clusterings. The scenario that each base clusterer processes only a part of the observations is called the \textit{observation-distributed scenario} \cite{strehl2003cluster,wang2009bayesian} or \textit{crowdclustering} \cite{gomes2011crowdclustering,yi2012crowdclustering}.  The main difference between MTC and the crowdclustering is that the crowdclustering assumes that all parts of observations are sampled from the same underlying distribution while MTC does not assume so. But, we have to note that several cluster ensemble techniques can be adapted to MTC, such as \cite{wang2009bayesian,gomes2011crowdclustering,yi2012crowdclustering},\cite{wang2011nonparametric}\footnote{Best Student Paper Award of \textbf{SDM-2011}}. Still, to our knowledge, none of the cluster ensembles can both hold convexity and be constructed on discriminative clusterings.

\section{Bayesian Framework of Discriminative Multitask Clustering}\label{sec:prob}
Suppose there are $m$ clustering tasks. The $i$-th task consists of $n_i$ unlabeled observations $\left\{\x_j^i\right\}_{j=1}^{n_i}$, $\x^i_j\in\mathbb{R}^d$. We cluster each task to the same number of classes, denoted as $C$ with $C\ge 2$. The prediction function of the $c$-th class for the $i$-th task is defined as $f^{i}_c(\x^i) = \w_{i,c}^T\x^i$, where $\w_{i,c}$ is the parameter of $f^i_c$, and where we have omitted the bias term $b_{i,c}$ in $f^i_c$ for simplicity. The observation $\x^i$ is assigned to the $c^{\star}$-th class, if $c^{\star} = \arg\max_{c} f^i_{\w_{i,c}}$ holds.
Note that the reason why we assume all tasks have the same number of classes is clarified as follows. 1) In practice, the related tasks tend to share similar structure. 2) We can easily extend this assumption to the scenario that the tasks have different number of classes by extending the prior (Eq. (\ref{eq:prob_feature})) from one-class-versus-one-class correlation to one-class-versus-all-classes correlation. For clarity, we use a more strict assumption.

For a $C$ class clustering problem, the discriminative clustering algorithm models $p\left(y|\x;\{\w_{c}\}_{c=1}^{C}\right)$, where $y\in\{ 1,2,\ldots,C\}$. We further extend $y$ to a $C$ dimensional indicator vector $\bar{\y}$, i.e. $\bar{\y} = \left[\by_1,\ldots,\by_C \right]$, where the label vector $\bar{\y}$ takes 1 for the $k$-th element and $-\frac{1}{C-1}$ for the others when $y=k$.
 For instance, if $\x$ falls into the first class, then $\bar{\y} = [1,-\frac{1}{C-1},\ldots,-\frac{1}{C-1}]$. This coding method is a common strategy in the multiclass problems, such as $k$-means. Note that $\bby$ is a row vector. Here, a set $\B_{\bby}$ is defined for all possible $\bar{\y}$, i.e.
  $\B_{\bby} = \big\{[1,-\frac{1}{C-1},\ldots,-\frac{1}{C-1}], [-\frac{1}{C-1},1,\ldots,-\frac{1}{C-1}],\ldots,$ $[-\frac{1}{C-1},-\frac{1}{C-1},\ldots,1] \big\}$.

 For a $m$-task MTC problem, we denote $\W_c = [\w_{1,c},\ldots,\w_{m,c}]$, $ \mathbf{X}^i=\left[\x_1^i,\ldots,\x_{n_i}^i\right]$, and
$\bbY^i = [(\bby_1^i)^T,\ldots,(\bby_{n_i}^i)^T]^T$.
  We try to optimize $\left\{\W_c\right\}_{c=1}^{C}$ under the Bayesian framework: The \textit{maximum a posteriori} estimation of $\left\{\W_c \right\}_{c=1}^{C}$ is formulated as
  \begin{eqnarray}
 \label{eq:prob2}
&&\max_{ \left\{\W_c \right\}_c, \{\bbY^i\}_i} p\left( \left\{\W_c \right\}_c, \{\bbY^i\}_i\Big| \{\mathbf{X}^i\}_i  \right) \nonumber\\
= && \max_{ \left\{\W_c \right\}_c, \{\bbY^i\}_i} p(\left\{\W_c \right\}_c) p\left( \{\bbY^i\}_i   \Big| \{\mathbf{X}^i\}_i , \left\{\W_c \right\}_c \right).
 \end{eqnarray}
Eq. (\ref{eq:prob2}) contains two parts. The first part $ p(\left\{\W_c \right\}_c)$ is a prior that defines the task relationship. The second part is a discriminative clustering model that covers all tasks. How to specify the prior and the discriminative model is the central problem.

 Now, we make four probabilistic assumptions on problem (\ref{eq:prob2}) for balancing the difficulty of solving DMTC and the effectiveness of DMTC.

 a) \textit{Class evenness assumption.} We assume that the empirical label marginal distribution $p(y)$ in each task is known and distributes evenly. This assumption has been adopted by many discriminative clustering algorithms, such as the class balance constraint assumption in MMC \cite{xu2005maximum,zhang2012linearithmic} and the maximal entropy assumption \cite{gomes2010discriminative}.  We prefer the class balance constraint assumption in \cite{zhang2012linearithmic} since it can simplify the mathematical form of (\ref{eq:prob2}) and is tunable. The constraint set $\mathcal{B}^{i}$ is defined as:
 \begin{eqnarray}
\B^i\triangleq\left\{ \bbY^i \bigg|\bigg\{\begin{array}{l} -\frac{l_{i,c}}{C-1} \le \frac{ \mathbf{1}^T_{n_i}\coly_c^i }{n_i} \le l_{i,c} ,\forall c = 1,\ldots,C, \\
\bby^i_j\in\B_{\bby}, \mbox{ }\forall j = 1,\ldots, n_i.\end{array}\right\}
 \label{eq:m3c_settt}
 \end{eqnarray}
 where $\coly^i_c=[\by^i_{1,c},\ldots,\by^i_{n_i,c}]^T$ denotes the $c$-th column of $\bbY^i$ and  $\{\{ l_{i,c}\}_{c=1}^{C}\}_{i=1}^{m}$ are user defined parameters that control the class balance. The constraint $-\frac{l_{i,c}}{C-1} \le \frac{ \mathbf{1}^T_{n_i}\coly_c^i }{n_i} \le l_{i,c} $ specifies the class evenness of the $c$-th class, while the constraint $\bby^i_j\in\B_{\bby}$ commands that $\Y^i$ must be a legal indicator matrix. This constraint set means that the indicator matrices who violate the constraints have 0 probability to appear, while the matrices who obey the constraints have an equal chance to appear.
As will be shown in the experimental section, a correct class balance assumption is very important to the success of DMTC. It not only can help DMTC detect a reasonable label pattern but also can prevent the interference of the outliers. If we know the class distribution, we can set $l_{l,c}$ to a value that is around ${ \mathbf{1}_{n_i}^T{{\y^{\star}}}_c^i }/{n_i}$ where ${{\y^{\star}}}_c^i$ is the $c$-th column of the ground truth label matrix of the $i$-th task, otherwise, we can just set all $l_{l,c}$ to the same empirical value.

b) \textit{Multivariate Gaussian prior assumption.} The prior defines what to share in MTC. In this paper, we follow Zhang and Yeung's formulation \cite[equation 2]{zhang2010convex} for the multivariate Gaussian prior.
 \begin{eqnarray}
 \label{eq:prob_feature}
 p(\{\W_c\}_c) \propto  \prod_{c=1}^{C} \left(q(\W_c) \prod_{i=1}^{m} \mathcal{N}\left( \w_{i,c} |  \mathbf{0}_{d},\sigma_1^2\mathbf{I}_{d} \right)\right)
 \end{eqnarray}
 where $\mathcal{N}\left(   \mathbf{A},\mathbf{B} \right)$ is a multivariate normal distribution with $\mathbf{A}$ and $\mathbf{B} $ as the mean and covariance matrix respectively, and $q(\W_c)$ is a distribution that the rows or columns of $\W_c$ are independent Gaussians. See (\ref{eq:prob_qf}) and (\ref{eq:prob5}) below for the definition of $\W_c$. As will be shown later, $\mathcal{N}\left( \w_{i,c} |  \mathbf{0}_{d},\sigma_1^2\mathbf{I}_{d} \right)$ plays a regularization role on the task-specific model $\w_{i,c},i = 1,\ldots,m$. Note that restricting all tasks have the same covariance $\sigma_1^2\mathbf{I}_{d}$ might be too tight. In practice, we can use different covariances for different tasks.

   In this paper, we consider two kinds of $q(\W_c)$. The first kind 
defines a shared feature representation:
   \begin{eqnarray}
 \label{eq:prob_qf}
q_f(\W_c) =\frac{\exp\left( -\frac{1}{2}\mbox{tr}(\W_c^T \bm\DD^{-1}\W_c) \right)}{(2\pi)^{md/2}|\DD|^{d/2}}
 \end{eqnarray}
 where $\DD$ is a covariance matrix that models the relationships between the features.
 The second kind follows Zhang and Yeung's formulation \cite[equation 2]{zhang2010convex}, which defines the relationship between the tasks:
  \begin{eqnarray}
 \label{eq:prob5}
q_r(\W_c) =\frac{\exp\left( -\frac{1}{2}\mbox{tr}(\W_c\bm\Omega^{-1} \W_c^T) \right)}{(2\pi)^{md/2}|\bm\Omega|^{m/2}}
 \end{eqnarray}
 where $\mathbf{\bm\Omega}$ is the covariance matrix that models the relationships between the task-specific models $\w_{i,c}$.

c) \textit{Task independence assumption.} We assume that when $\left\{\W_c \right\}_c$ is sampled from the prior distribution, the tasks are mutually independent:
  \begin{eqnarray}
 \label{eq:prob3}
&& p\left( \{\bbY^i\}_i   \Big| \{\mathbf{X}^i\}_i , \left\{\W_c \right\}_c \right)  =\prod_{i=1}^{m}p\left( \bbY^i | \mathbf{X}^i, \{\w^i_c\}_{c} \right)\nonumber\\
=&& \prod_{i=1}^{m}\prod_{c=1}^{C}p\left( \coly^i_c | \mathbf{X}^i, \w^i_c \right)
= \prod_{i=1}^{m}\prod_{j=1}^{n_i}\prod_{c=1}^{C}p\left( \by^i_{j,c} | \mathbf{x}^i_j, \w^i_c \right).
 \end{eqnarray}
With this assumption, we can incorporate any advanced binary-class discriminative clustering algorithm into $p\left( \coly^i_c | \mathbf{X}^i, \w^i_c \right)$ without modifying the clustering algorithm significantly.

d) \textit{Gaussian assumption on the discriminative clustering model.} We assume $p\left( \by^i_{j,c} | \mathbf{x}^i_j, \w^i_c \right)$ in (\ref{eq:prob3}) is Gaussian:
  \begin{eqnarray}
 \label{eq:prob_uu}
p\left( \by^i_{j,c} | \mathbf{x}^i_j, \w_{i,c} \right) = \mathcal{N}\left( \by^i_{j,c}  | \w_{i,c}^T\mathbf{x}^i_j, \sigma_2^2 \right).
 \end{eqnarray}
 This assumption makes the discriminative clustering a regression problem but not a classification problem, which might not be the real case since $\by^i_{j,c} \in\{-\frac{1}{C-1},1\}$ is a discrete variable. However, it is known that even in the supervised classification problem, if we set problem (\ref{eq:prob3}) with a non-Gaussian likelihood, the computations of predictions are analytically intractable \cite[page 39]{rasmussen2006gaussian}. Moreover, the regression based classifiers have been widely adopted, such as least-squares SVM. 

\section{Convex Discriminative Multitask Feature Clustering}\label{sec:DMTFC}
In this section, we will introduce the convex objective function of the proposed DMTFC.

Substituting Eqs. (\ref{eq:m3c_settt})-(\ref{eq:prob_qf}), (\ref{eq:prob3}) and (\ref{eq:prob_uu}) into problem (\ref{eq:prob2}) and taking the negative logarithm of (\ref{eq:prob2}) can derive the following objective function:
\begin{eqnarray}
 \label{eq:obj_dftfc}
\min_{\left\{\Y^i\in\mathcal{B}^i\right\}_{i=1}^{m}}  && \min_{\left\{\W_{c}\right\}_{c=1}^{C}}\min_{\DD} \sum_{c=1}^{C}\bigg(\frac{\lambda_1}{2}\mbox{tr}\left(\W_{c}^T\W_{c}\right)\nonumber\\
&& + \frac{\lambda_2}{2}  \mbox{tr}\left(\W_{c}^T\DD^{-1} \W_{c}\right)  + \frac{d\lambda_2}{2}\ln |\DD| \nonumber\\
&& + \sum_{i=1}^{m}\frac{1}{n_i}\sum_{j=1}^{n_i}\left( \by_{j,c}^i -\w_{i,c}^T\x_j^i\right)^2\bigg)
 \end{eqnarray}
 where $\lambda_1 $ and $\lambda_2 $ are two tunable regularization parameters that are related to $\sigma_1$ and $\sigma_2$.

Problem (\ref{eq:obj_dftfc}) is an \textit{NP}-complete mixed integer matrix optimization problem. First, $\Y^i$ is an integer matrix variable, which will cause the problem \textit{NP}-complete. Second, even if $\Y^i$ is known, problem (\ref{eq:obj_dftfc}) is still a minimization of a non-convex function, since $\ln |\DD|$ is a concave function. In this section, we will relax (\ref{eq:obj_dftfc}) to a convex optimization problem that should be convex with respect to both the objective function and the constraints \cite{boyd2004convex}.

In respect of the objective function, we replace $\ln |\DD|$ by the following convex constraint set:
 \begin{eqnarray}
  \label{eq:obj_D}
\D = \{ \DD | \DD\in \mathbb{R}^{d\times d},\DD \succeq 0,\tr(\DD) = 1\}
 \end{eqnarray}
 which results in the following MIP problem:
   \begin{eqnarray}
 \label{eq:obj_dftfc2}
&&\min_{\left\{\Y^i\in\mathcal{B}^i\right\}_{i=1}^{m}}   \min_{\left\{\W_{c}\right\}_{c=1}^{C}}\min_{\DD\in\D} \sum_{c=1}^{C}\bigg( \frac{\lambda_2}{2}  \mbox{tr}\left(\W_{c}^T\DD^{-1} \W_{c}\right)   \nonumber\\
&&   +  \frac{\lambda_1}{2}\mbox{tr}\left(\W_{c}^T\W_{c}\right)  +   \sum_{i=1}^m\frac{1}{n_i}\sum_{j=1}^{n_i}\left( \by_{j,c}^i -\w_{i,c}^T\x_j^i\right)^2  \bigg).
 \end{eqnarray}
 We can see that problem (\ref{eq:obj_dftfc2}) is quite similar with \cite[Theorem 1]{argyriou2008convex} except that (\ref{eq:obj_dftfc2}) is a regularized multiclass problem with label $\Y^i$ as an integer matrix variable.

 In respect of the constraints, we will construct a convex hull \cite{boyd2004convex} on $\mathcal{B}^i$ as in \cite{li2009tighter,zhang2012linearithmic}. Specifically,
 fixing $\{\Y^i\}_{i=1}^{m}$ and $\DD$, problem (\ref{eq:obj_dftfc2}) is formulated as:
  \begin{eqnarray}
 \label{eq:obj_dftfc3}
  \sum_{c=1}^{C} \Bigg(\min_{\W_{c}} && \frac{\lambda_1}{2}\mbox{tr}\left(\W_{c}^T\W_{c}\right)  + \frac{\lambda_2}{2}  \mbox{tr}\left(\W_{c}^T\DD^{-1} \W_{c}\right)    \nonumber\\
&&\qquad+  \sum_{i=1}^m\frac{1}{n_i}\sum_{j=1}^{n_i}\left( \by_{j,c}^i -\w_{i,c}^T\x_j^i\right)^2 \Bigg)
 \end{eqnarray}
 where the problems in the big brackets are mutually independent. We rewrite the problem in the big brackets in the constrained form as follows:
  \begin{eqnarray}
 \label{eq:obj_whatwhat}
 \min_{\W_{i}} && \frac{\lambda_1}{2}\mbox{tr}\left(\W_{c}^T\W_{c}\right) + \frac{\lambda_2}{2}  \mbox{tr}\left(\W_{c}^T\DD^{-1} \W_{c}\right)\nonumber\\
  &&\qquad +  \sum_{i=1}^m\frac{1}{n_i}\sum_{j=1}^{n_i}\left(\xi^i_{j,c}\right)^2 \\
\mbox{s.t. }&& \by_{j,c}^i  -\w_{i,c}^T\x_j^i =\xi^i_{j,c}, \forall i = 1,\ldots,m,\forall j = 1,\ldots,n_i.\nonumber
 \end{eqnarray}
  According to the Karush-Kuhn-Tucker conditions, the dual form of problem (\ref{eq:obj_whatwhat}) can be written as:
  \begin{eqnarray}
 \label{eq:obj_dftfc4}
\max_{\pmb\alpha_c}\sum_{j=1}^{n_i}\alpha_{j,c}^i \by_{j,c}^{i} -\frac{1}{2}{\bm\alpha_c}^T\wK_{\scriptsize\mbox{F}}\bm\alpha_c
 \end{eqnarray}
where $\bm\alpha_c = [\alpha^1_{1,c},\ldots,\alpha^{m}_{n_m,c}]^T$ are the dual variables, $\wK_{\scriptsize\mbox{F}} = \K_{\scriptsize\mbox{F}} + \frac{1}{2}\bm\Lambda$ with $\bm\Lambda$ as the diagonal matrix whose diagonal element equals to $n_i$ if the corresponding observation belongs to the $i$-th task, and $\K_{\scriptsize\mbox{F}}$ denoted as the multitask-kernel matrix for feature learning which is defined as:
     \begin{equation}
 \label{eq:obj_KMTF}
K_{\scriptsize\mbox{F}}\left(\x^{i_1}_{j_1},\x^{i_2}_{j_2}\right) = {\x^{i_1}_{j_1}}^T\DD(\lambda_1\DD+\lambda_2\I_d)^{-1}{\x^{i_2}_{j_2}}\left<\mathbf{e}_{i_1},\mathbf{e}_{i_2} \right>.
 \end{equation}
 $\W_{c}$ is obtained as:
   \begin{eqnarray}
 \label{eq:obj_WC}
\W_{c} =\sum_{i=1}^m\sum_{j=1}^{n_i}\alpha^i_{j,c}\DD\left( \lambda_1\DD+ \lambda_2\I_d \right)^{-1}\x^i_j\mathbf{e}_{i}^T.
 \end{eqnarray}
 where $\mathbf{e}_{i}$ represents the $i$-th column of the identity matrix.
 Substituting (\ref{eq:obj_dftfc4}) back to problem (\ref{eq:obj_dftfc3}) and then substituting (\ref{eq:obj_dftfc3}) back to problem (\ref{eq:obj_dftfc2}) can get an equivalent optimization problem of (\ref{eq:obj_dftfc3}) as follows:
   \begin{equation}
 \label{eq:obj_dftfc6}
\min_{\left\{\Y^i\in\mathcal{B}^i\right\}_{i=1}^{m}}\min_{\DD\in\D}\max_{\{\pmb\alpha_c\}_{c=1}^{C}} \sum_{i,c,j}\alpha_{j,c}^i \by_{j,c}^{i} -\frac{1}{2}\sum_{c=1}^C\bm\alpha_c^T\wK_{\scriptsize\mbox{F}}\bm\alpha_c.
 \end{equation}
 Because the second term of problem (\ref{eq:obj_dftfc6}) is irrelevant to the integer matrix variable $\Y^i$,
it is easy to see that the following problem learns a lower bound of problem (\ref{eq:obj_dftfc6}):
     \begin{eqnarray}
 \label{eq:obj_dftfc7}
&& \min_{\DD\in\D}\max_{\{\pmb\alpha_c\}_{c=1}^{C}}     \bigg\{ \max_{\{\theta_i\}_{i=1}^m} \sum_{i=1}^m\theta_i -\frac{1}{2}\sum_{c=1}^C\bm\alpha_c^T\wK_{\scriptsize\mbox{F}}\bm\alpha_c \\
&&\mbox{s.t.}    \theta_i \leq \sum_{c=1}^C\sum_{j=1}^{n_i}\alpha_{j,c}^i \by_{j,c}^{i} , \forall i=1,\dots,m, \forall k: \Y^{i}_k \in \B^i  \bigg\}.\nonumber
 \end{eqnarray}
Reformulating the problem in the braces of (\ref{eq:obj_dftfc7}) to its dual can get the following equivalent problem:
      \begin{eqnarray}
 \label{eq:obj_dftfc8}
 \min_{\DD\in\D}  && \max_{\{\pmb\alpha_c\}_{c=1}^{C}}\min_{\left\{\bm\mu^i\in\mathcal{M}^i\right\}_{i=1}^{m}}
-\frac{1}{2}\sum_{c=1}^{C}\bm\alpha_c^T\wK_{\scriptsize\mbox{F}}\bm\alpha_c\nonumber\\
&& + \sum_{i=1}^m\sum_{c=1}^C\sum_{j=1}^{n_i}\alpha_{j,c}^i \sum_{k:\Y^{i}_k \in \B^i}\mu^i_k\by^i_{k,j,c}
 \end{eqnarray}
where $\by^i_{k,j,c}$ is the element of $\Y^i_k$ at the $j$-th row and $c$-th column, and $\mathcal{M}^i$ is defined as $\mathcal{M}^i = \left\{\bm\mu^i| 0\le\mu^i_k\le1,\sum_{k :\Y_{k}^{i}\in\B^i}\mu^i_k = 1 \right\}$. If we denote $\wB^i = \left\{ \wY^i \Big| \wY^i = \sum_{k: \Y_{k}^{i}\in\B^i}\mu^i_k\Y^i_k, \bm\mu^i\in\mathcal{M}^i \right\}$, according to \cite[page 24]{boyd2004convex}, $\wB^i$ is the convex hull of $\B^i$ which is the tightest convex relaxation of $\B^i$. Note that the optimization order of $\{\bm\mu,\DD,\bm\alpha\}$ is exchangeable.

Writing the objective function in (\ref{eq:obj_dftfc8}) back to its primal form can derive the following equivalent convex optimization problem:
  \begin{eqnarray}
 \label{eq:obj_dftfc9}
&&\min_{\left\{\pmb\mu^i\in\mathcal{M}^i\right\}_{i=1}^{m}}     \min_{\left\{\W_{c}\right\}_{c=1}^{C}}\min_{\DD\in\D}
  \sum_{c=1}^{C}\Bigg(  \frac{\lambda_1}{2}\mbox{tr}\left(\W_{c}^T\W_{c}\right)  \nonumber\\
&&\quad+  \frac{\lambda_2}{2}\mbox{tr}\left(\W_{c}^T\DD^{-1} \W_{c}\right)\nonumber\\
&&\quad+ \sum_{i=1}^m\frac{1}{n_i}\sum_{j=1}^{n_i}\Bigg( \sum_{k:\Y_{k}^{i}\in\B^i}\mu^i_k\by_{k,j,c}^i -\w_{i,c}^T\x_j^i\Bigg)^2   \Bigg).
 \end{eqnarray}

\begin{thm}
  Problem (\ref{eq:obj_dftfc9}) is convex with respect to $\left\{\bm\mu^i\right\}_{i=1}^{m}$, $\left\{\W_c\right\}_{c=1}^{C}$, and $\DD$.
\end{thm}
\begin{proof}
Because $\left\{\mathcal{M}^i\right\}_{i=1}^{m}$, $\left\{\mathbb{R}^{d\times m}\right\}_{c=1}^{C}$ and $\D$ are all convex sets, their Cartesian product $ \mathcal{M}^1\times\ldots\times  \mathcal{M}^m\times  \mathbb{R}^{d\times m},\ldots,\mathbb{R}^{d\times m}\times \D$, i.e. the constraint, is also convex {\cite[page 38]{boyd2004convex}}, where $n=\sum_{i}n_i$.
   It is easy to see that the first and third terms of the objective function are convex by verifying that their \textit{Hessian} matrices are positive semidefinite \cite[page 71]{boyd2004convex}. The second term has been proved to be convex in \cite{argyriou2008convex}. Because the summation operation can preserve convexity, the objective function is convex. Therefore, problem (\ref{eq:obj_dftfc9}) is jointly convex with respect to all variables.
\end{proof}

Summarizing the aforementioned, problem (\ref{eq:obj_dftfc9}) is a convex relaxation of the original problem (\ref{eq:obj_dftfc}).  It has two equivalent forms (\ref{eq:obj_dftfc7}) and (\ref{eq:obj_dftfc8}). Problem (\ref{eq:obj_dftfc7}) is the objective function of DMTFC.

\section{Convex Discriminative Multitask Relationship Clustering}\label{sec:DMTRC}
In this section, we will introduce the convex objective function of the proposed DMTRC.

Substituting Eqs. (\ref{eq:m3c_settt}), (\ref{eq:prob_feature}), and (\ref{eq:prob5})-(\ref{eq:prob_uu}) into problem (\ref{eq:prob2}) and taking the negative logarithm of (\ref{eq:prob2}) can derive the following objective function:
\begin{eqnarray}
 \label{eq:obj_relationl}
\min_{\left\{\Y^i\in\mathcal{B}^i\right\}_{i=1}^{m}}   && \min_{\left\{\W_{c}\right\}_{c=1}^{C}}\min_{\pmb\Omega}  \sum_{c=1}^{C}\bigg( \frac{\lambda_1}{2}\mbox{tr}\left(\W_{c}\W_{c}^T\right) \nonumber\\
&&   + \frac{\lambda_2}{2}  \mbox{tr}\left(\W_{c}\bm\Omega^{-1} \W_{c}^T\right)   + \frac{m\lambda_2}{2}\ln |\bm\Omega|\nonumber\\
&& +\sum_{i=1}^{m}\frac{1}{n_i}\sum_{j=1}^{n_i}\left( \by_{j,c}^i -\w_{i,c}^T\x_j^i\right)^2    \bigg).
 \end{eqnarray}
 We can see that problem (\ref{eq:obj_relationl}) seems quite similar with problem (\ref{eq:obj_dftfc}) except that $\w_{i,c}$ and $\DD$ in (\ref{eq:obj_dftfc}) is replaced by $\w_{i,c}^T$ and $\bm\Omega$ respectively. However, essentially, what they learn is quite different.
  We can also observe that problem (\ref{eq:obj_relationl}) seems quite similar with \cite[equation 5]{zhang2010convex} except that (\ref{eq:obj_relationl}) is a multiclass problem and $\Y^i$ is an integer matrix variable. But this difference makes (\ref{eq:obj_relationl}) a hard MIP problem.
Observing the factors that cause problem (\ref{eq:obj_dftfc}) and problem (\ref{eq:obj_relationl}) non-convex are the same, we can use a similar convex relaxation procedure with (\ref{eq:obj_dftfc})'s for (\ref{eq:obj_relationl}). For the length limitation of the paper, we only report the main results.

The relaxed convex optimization problem of problem (\ref{eq:obj_relationl}) is formulated formally as follows:
  \begin{eqnarray}
 \label{eq:objxx}
&&\min_{\{\pmb\mu^i\in\mathcal{M}^i\}_{i=1}^{m}}    \min_{\{\W_{i}\}_{c=1}^{C}}\min_{\pmb\Omega\in\A}
 \sum_{c=1}^{C}\Bigg(   \frac{\lambda_1}{2}\mbox{tr}\left(\W_{c}\W_{c}^T\right)  \nonumber\\
&&\quad+ \frac{\lambda_2}{2}\mbox{tr}\left(\W_{c}\bm\Omega^{-1} \W_{c}^T\right) \nonumber\\
 && \quad+ \sum_{i=1}^m\frac{1}{n_i}\sum_{j=1}^{n_i}\Bigg( \sum_{k:\Y_{k}^{i}\in\B^i}\mu^i_k\by_{k,j,c}^i -\w_{i,c}^T\x_j^i\Bigg)^2   \Bigg).
 \end{eqnarray}
  where $\A$ is a convex constraint set defined as:
 \begin{eqnarray}
  \label{eq:obj_Omega}
\A = \{ \bm\Omega | \bm\Omega\in \mathbb{R}^{m\times m},\bm\Omega \succeq 0,\mathbf{tr}(\bm\Omega) = 1\}.
 \end{eqnarray}
  The proof of the convexity of problem (\ref{eq:objxx}) is similar with the proof of Theorem 1.
Problem (\ref{eq:objxx}) has two equivalent forms. The first one is written as:
     \begin{eqnarray}
 \label{eq:obj8}
 &&\min_{\pmb\Omega\in\mathcal{A}}\max_{\{\pmb\alpha_c\}_{c=1}^{C}}    \bigg\{ \max_{\{\theta_i\}_{i=1}^m} \sum_{i=1}^m\theta_i -\frac{1}{2}\sum_{c=1}^C\bm\alpha_c^T\wK_{\scriptsize\mbox{R}}\bm\alpha_c \\
\mbox{s.t.}   && \theta_i \leq \sum_{c=1}^C\sum_{j=1}^{n_i}\alpha_{j,c}^i \by_{j,c}^{i},\forall i = 1,\ldots,m,  \forall k: \Y^{i}_k \in \B^i  \bigg\}.\nonumber
 \end{eqnarray}
where $\wK_{\scriptsize\mbox{R}} = \K_{\scriptsize\mbox{R}} + \frac{1}{2}\bm\Lambda$ with $\K_{\scriptsize\mbox{R}}$ denoted as the multitask-kernel matrix for relationship learning which is defined as \cite{zhang2010convex}:
     \begin{equation}
 \label{eq:obj_KMTR}
K_{\scriptsize\mbox{R}}(\x^{i_1}_{j_1},\x^{i_2}_{j_2}) = \e_{i_1}^T\bm\Omega(\lambda_1\bm\Omega+\lambda_2\I_m)^{-1}\e_{i_2}\left<\x^{i_1}_{j_1},\x^{i_2}_{j_2}\right>.
 \end{equation}
 We also obtain $\W_c$ as:
   \begin{eqnarray}
 \label{eq:obj_WC2}
\W_c =\sum_{i=1}^{m}\sum_{j=1}^{n_i}\alpha^i_{j,c}\x^i_j\e_i^T \bm\Omega \left( \lambda_1\bm\Omega+ \lambda_2\I_m \right)^{-1}.
 \end{eqnarray}
The second equivalent form is written as:
      \begin{eqnarray}
 \label{eq:obj9}
 \min_{\pmb\Omega\in\mathcal{A}}  && \max_{\{\pmb\alpha_c\}_{c=1}^{C}}\min_{\left\{\bm\mu^i\in\mathcal{M}^i\right\}_{i=1}^{m}}
-\frac{1}{2}\sum_{c=1}^{C}\bm\alpha_c^T\wK_{\scriptsize\mbox{R}}\bm\alpha_c\nonumber\\
&& + \sum_{i=1}^m\sum_{c=1}^C\sum_{j=1}^{n_i}\alpha_{j,c}^i \sum_{k:\Y^{i}_k \in \B^i}\mu^i_k\by^i_{k,j,c}
 \end{eqnarray}

Summarizing the aforementioned, problem (\ref{eq:objxx}) is a convex relaxation of the original problem (\ref{eq:obj_relationl}). It has two equivalent forms (\ref{eq:objxx}) and (\ref{eq:obj9}). Problem (\ref{eq:obj8}) is the objective function of DMTRC.

\section{Optimization Algorithm}\label{sec:solution}
In this section, we are to solve DMTFC (\ref{eq:obj_dftfc7}) and DMTRC (\ref{eq:obj8}) in a uniform framework. This framework utilizes the fact that there are only two different points between them: 1) the multitask kernel functions are different, see Eqs. (\ref{eq:obj_KMTF}) and (\ref{eq:obj_KMTR}); 2) the convex sets $\D$ and $\A$ are different, see Eqs. (\ref{eq:obj_D}) and (\ref{eq:obj_Omega}). To facilitate the mathematical representation, we write (\ref{eq:obj_dftfc7}) and (\ref{eq:obj8}) as the following uniform objective:
      \begin{eqnarray}
 \label{eq:objzzz}
  &&\max_{\{\pmb\alpha_c\}_{c=1}^{C}} \min_{\ZZ\in\Z}   \max_{\{\theta_i\}_{i=1}^m} \sum_{i=1}^m\theta_i -\frac{1}{2}\sum_{c=1}^C\bm\alpha_c^T\wK\bm\alpha_c \\
&&\mbox{s.t.}    \theta_i \leq \sum_{c=1}^C\sum_{j=1}^{n_i}\alpha_{j,c}^i \by_{j,c}^{i} , \forall i=1,\dots,m, \forall k: \Y^{i}_k \in \B^i  .\nonumber
 \end{eqnarray}
 where $\ZZ$ stands for $\DD$ in (\ref{eq:obj_dftfc7}) or $\bm\Omega$ in (\ref{eq:obj8}), $\Z$ stands for $\D$ in (\ref{eq:obj_dftfc7}) or $\A$ in (\ref{eq:obj8}), and $\wK$ stands for $\wK_{\scriptsize\mbox{F}}$ in (\ref{eq:obj_KMTF}) or $\wK_{\scriptsize\mbox{R}}$ in (\ref{eq:obj_KMTR}).

Due to the length limitation of the paper, we present the optimization algorithm briefly as follows, leaving the detailed derivation in the supplemental material which is available at \url{http://sites.google.com/site/zhangxiaolei321/}.

 The solution framework is an alternating method.
  First, it decomposes the unsupervised problem (\ref{eq:objzzz}) to a serial supervised multiclass MTL problem by the cutting-plane algorithm (CPA) \cite{kelley1960cutting} and the extended level method (ELM) \cite{xu2009extended,yang2011efficient}, where the decomposition algorithm can be seen as a multitask extension of the SVR-M3C algorithm \cite{zhang2012linearithmic}. Then, it solves each supervised multiclass MTL problem in an alternating way, which decomposes the multiclass MTL to a serial supervised single-task regression problems eventually. Note that the difference of the optimization procedure between DMTFC and DMTRC only appears in the supervised learning in Section \ref{subsubsection:3}.

 \subsection{Optimizing (\ref{eq:objzzz}) Via Cutting-plane Algorithm}\label{subsec:cpa}

Because the number of the constraints in problem (\ref{eq:objzzz}) is exponential large with respect to $n$, directly optimizing (\ref{eq:objzzz}) is impossible when the data set contains over dozens of examples. Hence, we adopt CPA \cite{kelley1960cutting} to solve it approximately. CPA iterates the following two steps. The first step is to solve the following reduced cutting plane subproblem:
      \begin{eqnarray}
 \label{eq:obj13}
  &&\max_{\{\pmb\alpha_c\}_{c=1}^C} \min_{\ZZ\in\Z} \max_{\{\theta_i\}_{i=1}^m} \sum_{i=1}^m\theta_i -\frac{1}{2}\sum_{c=1}^C\bm\alpha_c^T\wK\bm\alpha_c \\
&&\mbox{s.t.}    \theta_i \leq \sum_{c=1}^C\sum_{j=1}^{n_i}\alpha_{j,c}^i \by_{j,c}^{i} , \forall i=1,\dots,m, \forall k: \Y^{i}_k \in \cY^i  \bigg\}.\nonumber
 \end{eqnarray}
      where $\cY^i\subset\B^i$ represents the pool of the most violated constraints,  The second step is to calculate the most violated constraint, denoted as $\{\Y_{|\cY^i|+1}^i\}_{i=1}^m$, by solving the following integer matrix optimization problem
\begin{equation}\label{eq:violated}
\min_{\Y_{|\cY^i|+1}^i}\sum_{c=1}^C\sum_{j=1}^{n_i}\alpha_{j,c}^i \by_{|\cY^i|+1,j,c}^{i}\mbox{ },\mbox{ } \forall i = 1,\ldots,m,
\end{equation}
 and then add $\Y_{|\cY^i|+1}^i$ to $\cY^i,\forall i = 1,\ldots,m$, respectively.
Thanks to the constraints on $\Y^i$ (defined in $\B^i$, i.e. Eq. (\ref{eq:m3c_settt})), the problem can be solved in time $\O\left(\sum_{i=1}^{m}Cn_i\log (Cn_i)\right)$, see \cite[Algorithm 6]{zhang2012linearithmic} for the algorithm.

\subsection{Optimizing (\ref{eq:obj13}) Via Extended Level Method}\label{subsec:elm}
Like the full problem (\ref{eq:objzzz}), the cutting-plane subproblem (\ref{eq:obj13}) also has an equivalent form:
      \begin{eqnarray}
 \label{eq:objxxxx15}
  \max_{\{\pmb\alpha_c\}_{c=1}^C} &&\min_{\ZZ\in\Z} \min_{\{\pmb\mu^i\in\mathcal{M}_{\cY}^i\}_{i=1}^m} -\frac{1}{2}\sum_{c=1}^{C}\bm\alpha_c^T\wK\bm\alpha_c\nonumber\\
&& + \sum_{i=1}^m\sum_{c=1}^C\sum_{j=1}^{n_i}\alpha_{j,c}^i \sum_{k:\Y^{i}_k \in \cY^i}\mu^i_k\by^i_{k,j,c}
 \end{eqnarray}
where $\mathcal{M}_{\cY}^i = \left\{\bm\mu^i| 0\le\mu^i_k\le1,\sum_{k =1}^{|\cY^i|}\mu^i_k = 1 \right\}$.

 Problem (\ref{eq:objxxxx15}) is a concave-convex optimization problem that is convex on $\bm\mu$ and $\ZZ$ and concave on $\bm\alpha$. We will optimize it via ELM \cite{xu2009extended} which is an efficient alternating method that aims to find the saddle point of the problem. ELM iterates the following two steps until convergence. The first step is to optimize $\{\bm\mu^i\}_{i=1}^m$ given fixed $\{\bm\alpha\}_{c=1}^C$ and $\ZZ$ by constructing a cutting-plane model on the problem. See the supplement for this complicated cutting-plane model. The second step is to optimize $\{\bm\alpha\}_{c=1}^C$ and $\ZZ$ together given fixed $\{\bm\mu^i\}_{i=1}^m$, which is formulated as follows:
 \begin{eqnarray}
    \label{eqn:xxxx}
    \min_{\ZZ\in\Z}\max_{\{\pmb\alpha_c\}_{c=1}^C}&\mbox{ }&-\frac{1}{2}\sum_{c=1}^{C}\bm\alpha_c^T\wK\bm\alpha_c\nonumber\\
&& + \sum_{i=1}^m\sum_{c=1}^C\sum_{j=1}^{n_i}\alpha_{j,c}^i \sum_{k:\Y^{i}_k \in \cY^i}\mu^i_k\by^i_{k,j,c}
 \end{eqnarray}

 Note that problem (\ref{eqn:xxxx}) is the dual form of a supervised MTL problem. The reason why we solve DMTC in the dual form but not primal form is because that we need the Lagrange parameter $\bm\alpha$ to solve problem (\ref{eq:violated}) but not only for introducing the nonlinear kernels.

 \subsection{Optimizing (\ref{eqn:xxxx}) Via Alternating Method}\label{subsubsection:3}

We adopt an alternating method that is similar with \cite{zhang2010convex} for problem (\ref{eqn:xxxx}), which iterates the following two steps until convergence.

 The first step is to optimize $\{\bm\alpha\}_{c=1}^C$ given fixed $\ZZ$, which is equivalent to the following problem:
\begin{equation}
    \label{eqn:bierende2}
   \sum_{c=1}^{C}\Bigg(\max_{\pmb\alpha_c}\sum_{i=1}^m\sum_{j=1}^{n_i}\alpha_{j,c}^i \sum_{k:\Y^{i}_k \in \cY^i}\mu^i_k\by^i_{k,j,c}-\frac{1}{2}\bm\alpha_c^T\wK\bm\alpha_c\Bigg)
\end{equation}
When $\ZZ$ is fixed, the terms in the brackets are mutually independently. Hence, we solve each term
independently, which is a supervised single-task regression problem, where the data from all tasks are considered as the data from a single task.

The second step is to optimize $\ZZ$ given fixed $\{\bm\alpha_c\}_{c=1}^{C}$, which is formulated as
 \begin{eqnarray}
    \label{eqn:bierendexx}
    \min_{\ZZ\in \mathcal{Z}}-\frac{1}{2}&&\sum_{c=1}^{C}\bm\alpha_c^T\wK\bm\alpha_c\nonumber\\
&& + \sum_{i=1}^m\sum_{c=1}^C\sum_{j=1}^{n_i}\alpha_{j,c}^i \sum_{k:\Y^{i}_k \in \cY^i}\mu^i_k\by^i_{k,j,c}
\end{eqnarray}
Note that $\wK$ is a function of $\ZZ$.

\textit{Specifying (\ref{eqn:bierende2}) and (\ref{eqn:bierendexx}) as a part of DMTFC:} We replace $\ZZ$ and $\Z$ in the equations by $\DD$ and $\D$ respectively. For (\ref{eqn:bierende2}), the multitask kernel $\wK$ should be specified by Eq. (\ref{eq:obj_KMTF}). The calculation of $\wK$ will be expensive when the dimension of the observation $d$ is large, since the time complexity of the matrix inversion in (\ref{eq:obj_KMTF}) is $\O(d^3)$ in the worst cases. For (\ref{eqn:bierendexx}), we can get the closed solution of $\DD$ as $\DD = \frac{\left(\sum_{c=1}^{C}\W_c\W_c^T\right)^{\frac{1}{2}}}{\mbox{tr}\left( \left( \sum_{c=1}^{C}\W_c\W_c^T \right)^{\frac{1}{2}} \right)}$
    where $\W_c$ is defined in (\ref{eq:obj_WC}).
    The derivation is analogous to \cite[Appendix 1]{argyriou2008convex}.

    \textit{Specifying (\ref{eqn:bierende2}) and (\ref{eqn:bierendexx}) as a part of DMTRC:} We replace $\ZZ$ and $\Z$ by $\bm\Omega$ and $\A$ respectively in the equations. For (\ref{eqn:bierende2}), $\wK$ should be specified by Eq. (\ref{eq:obj_KMTR}). The calculation of $\wK$ will be expensive when the task number $m$ is large, since the time complexity of the matrix inversion in (\ref{eq:obj_KMTR}) is $\O(m^3)$ in the worst cases. For (\ref{eqn:bierendexx}), we can get the closed solution of $\bm\Omega$ as $\bm\Omega = \frac{\left(\sum_{c=1}^{C}\W_c^T\W_c\right)^{\frac{1}{2}}}{\mbox{tr}\left( \left( \sum_{c=1}^{C}\W_c^T\W_c \right)^{\frac{1}{2}} \right)}$
    where $\W_c$ is defined in (\ref{eq:obj_WC2}).
    The derivation is analogous to \cite[equation 13]{zhang2010convex}.

\section{Learning With Nonlinear Kernels}\label{sec:kernel}
Incorporating the nonlinear feature mapping to DMTFC and DMTRC, we only need to modify their multitask kernel representations. Specifically, for DMTFC, we only need to modify Eq. (\ref{eq:obj_KMTF}) to $K_{\scriptsize\mbox{F}}\left(\x^{i_1}_{j_1},\x^{i_2}_{j_2}\right) = \e_{i_1}^T{\phi(\x^{i_1}_{j_1})}^T\DD(\lambda_1\DD+\lambda_2\I_d)^{-1}{\phi(\x^{i_2}_{j_2})}\e_{i_2}$ and modify Eq. (\ref{eq:obj_WC}) to $\W_c =\sum_{i}\sum_{j}\alpha^i_j\DD\left( \lambda_1\DD+ \lambda_2\I_d \right)^{-1}\phi(\x^i_j)\e_i^T$, where $\phi(\cdot)$ is the kernel-induced feature mapping. Because $\phi(\cdot)$ might be high dimensional or even infinite, such as the Radius-Basis-Function (RBF) kernel, we cannot calculate its representation accurately. Instead, we can use the kernel decomposition techniques, such as kernel principle component analysis or Cholesky decomposition, to get $\phi(\cdot)$ approximately and explicitly.
Similarly, for DMTRC, we only need to modify Eq. (\ref{eq:obj_KMTR}) to $K_{\scriptsize\mbox{R}}\left(\x^{i_1}_{j_1},\x^{i_2}_{j_2}\right) = \e_{i_1}^T\bm\Omega(\lambda_1\bm\Omega+\lambda_2\I_m)^{-1}\e_{i_2}K(\x^{i_1}_{j_1},\x^{i_2}_{j_2})$ and modify Eq. (\ref{eq:obj_WC2}) to $\W_c =\sum_{i}\sum_{j}\alpha^i_j\phi(\x^i_j)\e_i \bm\Omega \left( \lambda_1\bm\Omega+ \lambda_2\I_m \right)^{-1}$, where $K(\x,\y) = \left<\phi(\x),\phi(\y) \right>$.
Because DMTRC can incorporate nonlinear kernels implicitly via the kernel function $K$ while DMTFC needs to calculate the representation of the feature mapping $\phi(\cdot)$ explicitly with additional time and storage complexities of at least $\O(n^2)$. DMTRC is more efficient than DMTFC in kernel learning.

\section{Complexity Analysis}\label{sec:complexity}
Because the optimization algorithm can be seen as a technical combination of SVR-M3C \cite{zhang2012linearithmic}, MTFL \cite{argyriou2008convex}, and MTRL \cite{zhang2010convex}, where the outer two loops (i.e. Sections \ref{subsec:cpa} and \ref{subsec:elm}) is a multitask extension of SVR-M3C and the inner loop (i.e. Section \ref{subsubsection:3}) can be seen as a special case of the multiclass extensions of MTFL/MTRL, the overall time and storage complexities of the optimization algorithm are dominated by the most expensive algorithm between SVR-M3C and MTFL/MTRL. SVR-M3C has a time complexity of $\O(n\log n)$ and a storage complexity of $\O(n)$ \cite{zhang2012linearithmic}. It is also easy to observe that the worst case of MTFL has a time complexity of $\O(n^2+d^3)$ and a storage complexity of $\O(n^2)$, and that the worst case of MTRL has a time complexity of $\O(n^2+m^3)$ and a storage complexity of $\O(n^2)$. Hence, DMTFC is suitable to middle scale and low dimensional problems, while DMTRC is suitable to middle scale problems with small task numbers.
The main obstacle that hinders DMTFC and DMTRC from large scale problems is the time-demanding kernel calculation and matrix inversion in (\ref{eq:obj_KMTF}) and (\ref{eq:obj_KMTR}). To overcome it, dimension reduction techniques, sparse MTL techniques, distributed cluster ensembles and sparse kernel estimations might be helpful. But as will be shown in the experimental section, when the data size is large scale, the benefit of multitask clustering over single-task clustering will vanish. Finally, we do not think the complexity as a huge block that hinders them from practical use.

\section{Experiments}\label{sec:experiments}
In this section, we will compare the proposed DMTFC and DMTRC algorithms with 10 clustering algorithms on the UCI \textit{pendigits} toy dataset and two benchmark datasets -- \textit{multi-domain newsgroups} dataset and \textit{multi-domain sentiment} dataset. All experiments are run with MATLAB R2012b on a 2.40 GHZ 8-core Itel(R) Xeon(R) Server running Windows XP with 16 GB memory.

The competitive algorithms can be categorized to two classes. The first class are the Single Task Clustering (STC) algorithms. They are 1) K-Means (KM), 2) Kernel K-Means (KKM) with the RBF kernel, 3) Normalized Cut (NC) \cite{shi2002normalized} with the RBF kernel, 4) the Discriminative STC (DSTC) algorithm, 5) KM that groups all tasks into a single task (ALL KM), 6) ALL KKM, and 7) ALL NC, where DSTC is the single task version of our DMTRC. The DSTCs with the linear kernel and the RBF kernel are denoted as DSTC$_l$ and DSTC$_r$ respectively.
 The second class are the state-of-the-art MTC algorithms. They are 1) Learning the Shared Subspace for MTC (LSSMTC) \cite{gu2009learning}, 2) Learning a Spectral Kernel for MTC (LSKMTC) \cite{gu2011learning}, and 3) Multitask Bregman Clustering with Pairwise task regularization (MBC-P) \cite{zhang2011multitask}.  The experiments of the competitive algorithms are run exactly with the authors' experimental settings.

For our DMTFC and DMTRC, $\lambda_1$ and $\lambda_2$ are both searched from $\{2^{-10}, 2^{-8},\ldots,2^{-2} \}$, we make a strong assumption that we know the class distribution beforehand, so that $l_{i,c}$ in Eq. (\ref{eq:m3c_settt}) is set to $l_{l,c} ={ \mathbf{1}_{n_i}^T{{\y^{\star}}}_c^i }/{n_i}$ where ${{\y^{\star}}}_c^i$ is the $c$-th column of the ground truth label matrix $\Y^i$ of the $i$-th task. The DMTFC and DMTRC with the linear kernel are denoted as DMTFC$_l$ and DMTRC$_l$ respectively, and those with the RBF kernel are denoted as DMTFC$_r$ and DMTRC$_r$ respectively.

The kernel width of  all algorithms that work with the RBF kernel is searched from $\{2^{-2}, 2^{-1},2^{0}, 2^{1},2^{2}\}\cdot A$, where $A$ is the average Euclidean distance of the data. The data are normalized into the range of [0,1] in dimension. All computation time is recorded except that consumed on normalizing the dataset.
The datasets used in experiments are provided with labels. Therefore, the performance is evaluated as comparing the predicted labels with the ground truth labels using Normalized Mutual Information (NMI) \cite{strehl2003cluster}.




\subsection{Results on Pendigits Dataset}
In this subsection, the \textit{pendigits} dataset in the UCI machine learning repository is used as a toy dataset for capturing the main characteristics of the proposed DMTC algorithms. The pendigits dataset contains 10 hand written integer digits ranging from 0 to 9. It consists of 11256 observations and 16 attributes. Each digit consists of about 1100 observations. Although the pendigits dataset is a single task clustering problem, we generate a multitask clustering problem from it: First, we take $0,3,6,8,9$ as one group, and $1,2,4,5,7$ as the other group. Then, we repeatedly sample 20 observations from each digit in the first group for 3 times. Again, we do the same thing to the second group. Because each repeat forms a 5-class clustering task that contains 100 observations, we obtain 6 tasks in total, where Tasks 1, 2 and 3 are examples from the first group and Tasks 4, 5, and 6 are examples from the second group. Because the data are too small to cover the distributions of the digits, we can regard Tasks 1, 2 and 3 are relevant but not the same, so as to Tasks 4, 5, and 6. We also regard that Tasks 1, 2 and 3 are irrelevant to Tasks 4, 5, and 6.
A visualized example of the data distributions associated with the six tasks are shown in Fig. \ref{fig:wavefig_featurefusion}. We run three jobs on the six tasks. Job 1 is to cluster Tasks 1, 2, and 3. Job 2 is to cluster Tasks 4, 5, and 6. Job 3 is to cluster Tasks 1--6 together. For each MTC job, we repeat the experiment 30 times. For each single repeat, we also repeat the referenced algorithms 50 times and report the average results. For DMTFC$_r$, KPCA is used for getting $\phi(\x)$ explicitly. It retains the top 100 largest eigenvalues and their eigenvectors.


\begin{figure}[t]
\center
\scalebox{1}{
\centerline{\includegraphics[width=3in]{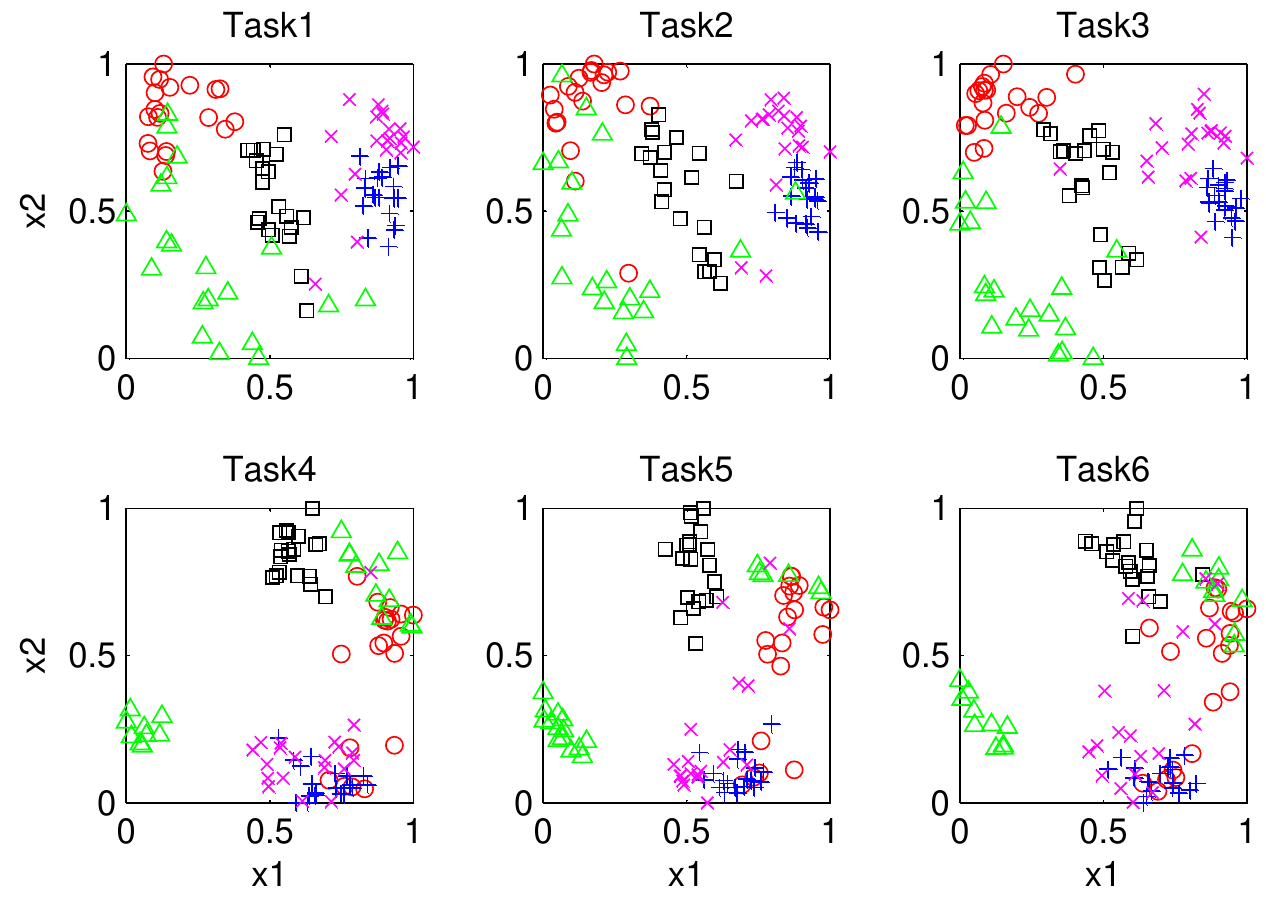}}}
\caption{{Visualization of the tasks on the pendigits data. The true labels are indicated by different colors and different symbols. PCA is used to generate the figure.} }
\label{fig:wavefig_featurefusion}
\end{figure}

\begin{figure}[t]
\scalebox{1}{
\centerline{\includegraphics[width=3.25in]{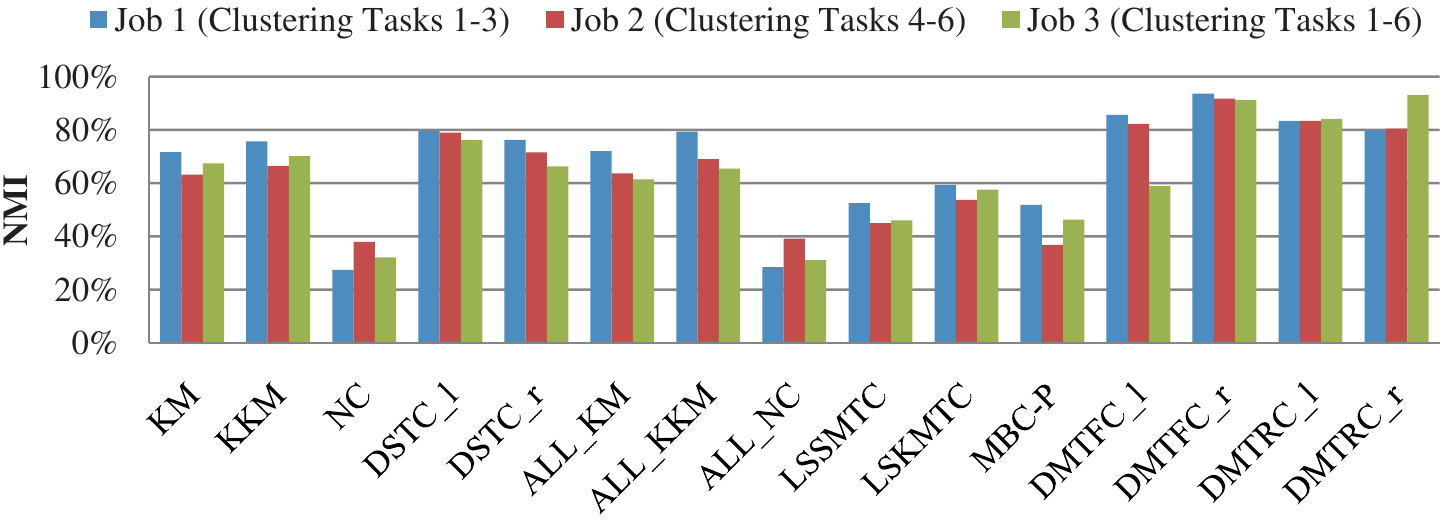}}}
\caption{{NMI comparison on the pendigits dataset.} }
\label{fig:wavefig_featurefusion2}
\end{figure}

\begin{figure}[t]
\scalebox{1}{
\centerline{\includegraphics[width=2.3in]{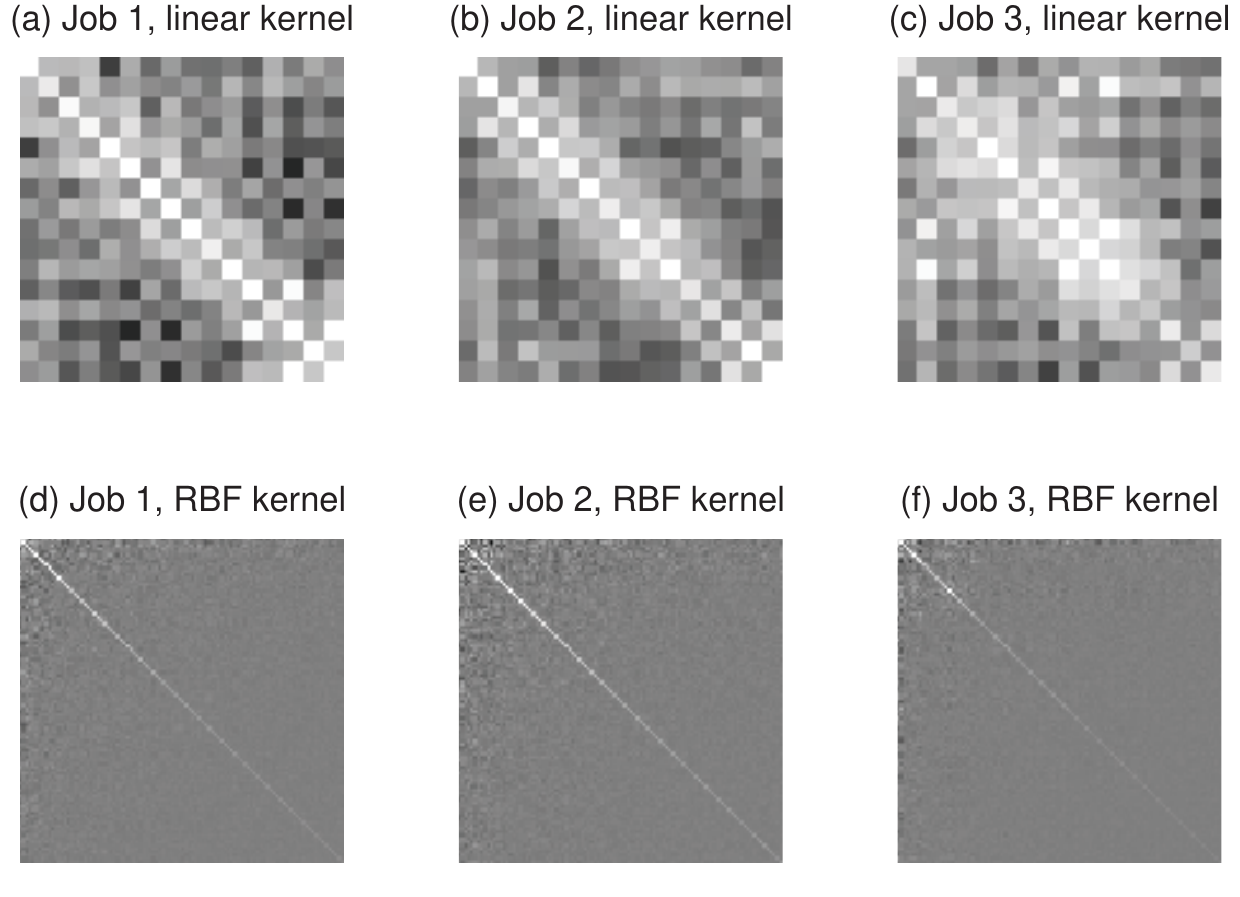}}}
\caption{{Visualization of the shared feature filter learned by DMTFC on the pendigits dataset (i.e. the learned covariance between the features, i.e. $\DD$). The more grey the grid is, the weaker the filter contributes to the new feature representation.} }
\label{fig:wavefig_featurefusion3}
\end{figure}

\begin{figure}[t]
\scalebox{1}{
\centerline{\includegraphics[width=2.3in]{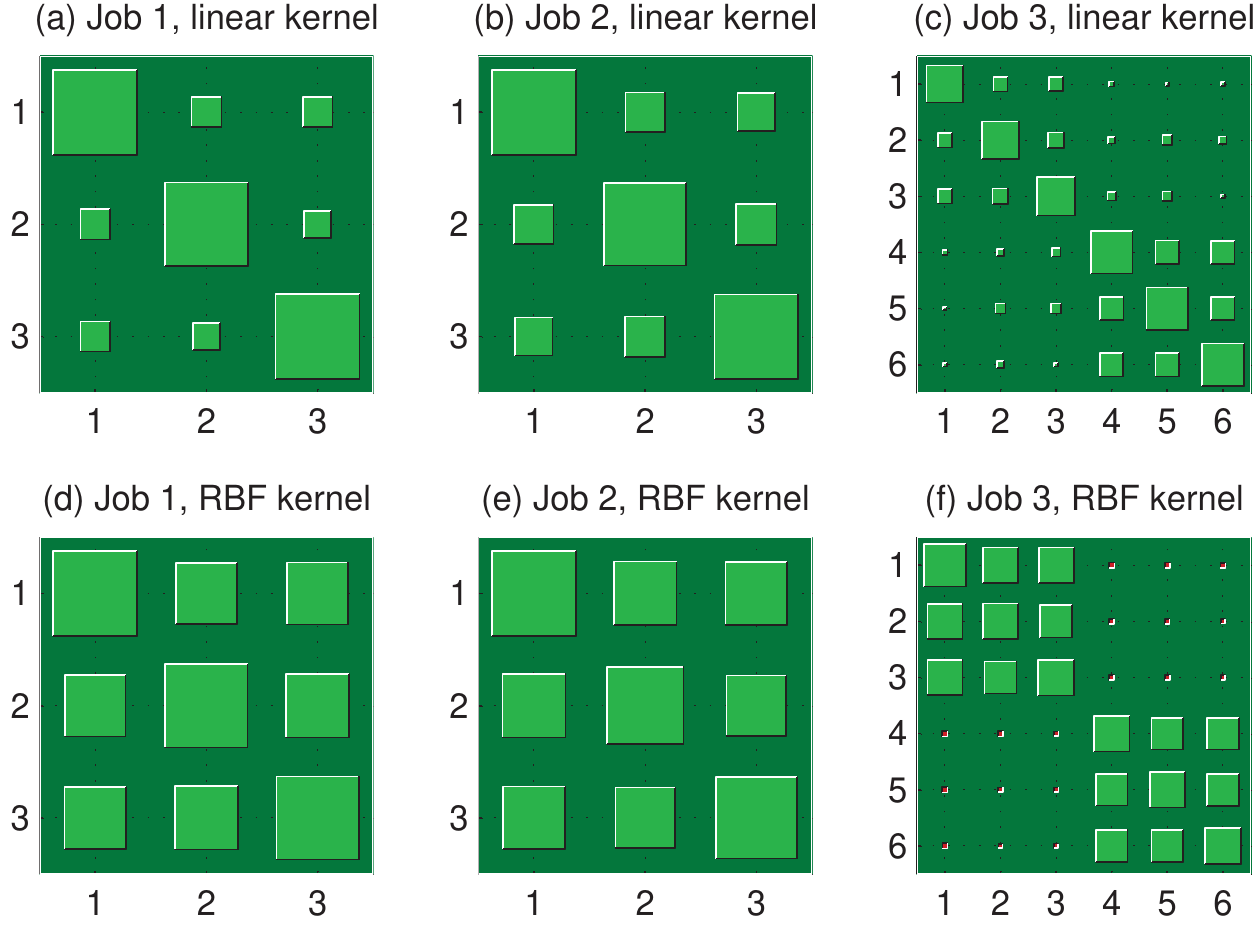}}}
\caption{{Hinton diagram of the task relationship learned by DMTRC on the pendigits dataset (i.e. the learned covariance between the task-specific models, i.e. $\bm\Omega$). The grid in green means the tasks are related. The grid in red means the tasks are reverse. The bigger the grid is, the more positive/negative the relationship is.} }
\label{fig:wavefig_featurefusion4}
\end{figure}

Fig. \ref{fig:wavefig_featurefusion2} shows the NMI comparison over the three jobs. From the figure, we can get the following interesting phenomena.
First, except for DMTFC$_l$, the proposed DMTC algorithms achieve higher NMIs than the referenced methods. This phenomenon demonstrates the effectiveness of the proposed MTC algorithms.
Second, except for DMTRC$_r$, the NMIs of all algorithms in Job 3 are lower than those in Jobs 1 and 2. This phenomenon is particularly apparent in DMTFC$_l$. It shows that the unrelated tasks or the reverse distributions worsen the clustering performance significantly. This phenomenon also shows that when the tasks are really related, learning a powerful feature representation is better than minimizing the distances between the task-specific models, but when the tasks are irrelevant, learning a feature representation forcibly is very harmful while learning the task relationship can avoid the negative transfer amazingly. To better explain this, we visualize $\DD$ and $\bm\Omega$ in Figs. \ref{fig:wavefig_featurefusion3} and \ref{fig:wavefig_featurefusion4} respectively. For DMTFC, in Figs. \ref{fig:wavefig_featurefusion3}a, \ref{fig:wavefig_featurefusion3}b, \ref{fig:wavefig_featurefusion3}d, \ref{fig:wavefig_featurefusion3}e, and \ref{fig:wavefig_featurefusion3}f,  the relationships of the features have been learned successfully by DMTFC. But in Fig. \ref{fig:wavefig_featurefusion3}c, DMTFC$_l$ fails in learning a common feature representation, i.e., most features are recognized as mutually independent. For DMTRC, in Fig. \ref{fig:wavefig_featurefusion4}, we can observe that DMTRC can capture the relationships of the tasks successfully no matter in Jobs 1 and 2 or in Job 3, which accounts for the immunity of DMTRC to the negative transfer. Note that this study has been conducted in many supervised MTL works, but to our knowledge, this is the first work that captures the task relationship successfully in the unsupervised learning scenario.
Third, the referenced MTCs do not achieve better NMIs than the STCs. One possible explanation for this is that the referenced MTCs suffer from local minima more seriously than the STCs.

 The above experiment assumes that the class distributions are known with all parameters $l_{i,c}$ setting to the ideal situation ${ \mathbf{1}_{n_i}^T{{\y^{\star}}}_c^i }/{n_i} = 0$. In this paragraph, we will investigate how the class evenness assumption affects the performance by setting all $\{\{l_{i,c}\}_{c=1}^{C}\}_{i=1}^{m}$ to the same value that is selected from $\{0,0.03,0.1, 0.2, 0.3\}$. The results are shown in Fig. \ref{fig:class_balance}. From the figure, we can observe the following phenomena: 1) In all settings, DMTC can benefit from joint training of all tasks except DMTFC$_l$. 2) Setting the class balance parameters to a value 0.03 that is slightly biased from the ideal situation can achieve even better performance, which means that if we select $l$ properly around the ideal value, the performance is guaranteed. 3) DMTC is sensitive to $l$, if parameter $l$ is set improperly, the performance will degrade dramatically. Hence, for DMTC's practical use, we should select $l$ carefully.

 \begin{figure}[t]
\scalebox{1}{
\centerline{\includegraphics[width=3.25in]{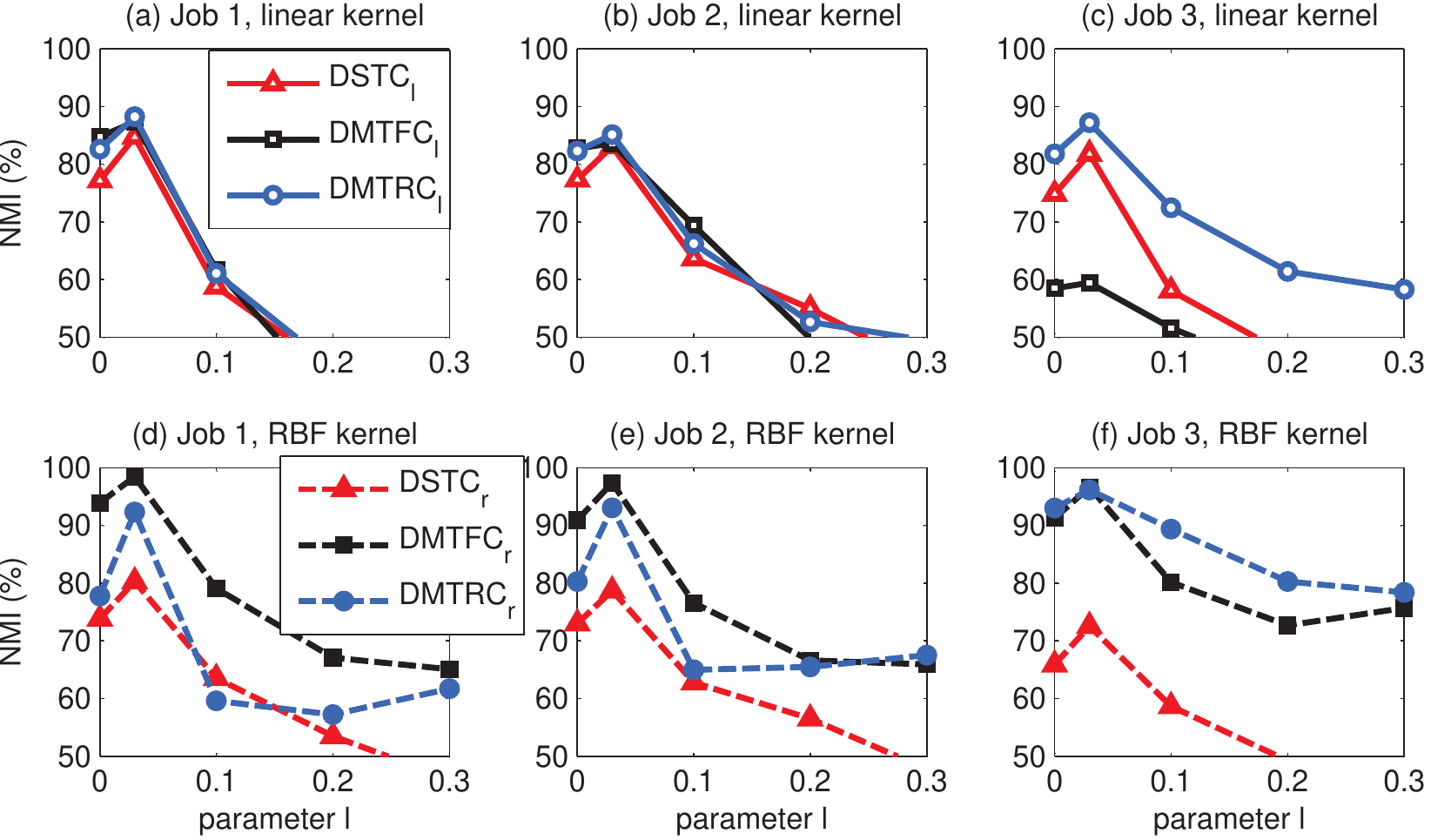}}}
\caption{{Clustering performance with respect to the class balance parameter $l$ on the pendigits dataset.} }
\label{fig:class_balance}
\end{figure}

\subsection{Complexity Analysis on Synthetic Dataset}
In this subsection, we will study the time complexities of DMTFC and DMTRC with respect to the number of examples of each task (i.e. $n$), feature dimension (i.e. $d$), and number of tasks (i.e. $m$) respectively. We generate each dimension of each class of each binary-class synthetic task from a Gaussian distribution, whose mean is sampled uniformly from $[0,1]$ and variance varies uniformly in $[0.5, 5]$.
The parameters of the proposed methods are as follows. Only linear kernel is considered. $\lambda_1=\lambda_2 = 2^{-10}$, $l=0$.

The time complexities with respect to $n$ are shown in Fig. \ref{fig:synthetic}a, where $d=3$ and $m=3$. The time complexities with respect to $d$ are shown in Fig. \ref{fig:synthetic}b, where $n=100$ and $m=3$. The time complexities with respect to $m$ are shown in Fig. \ref{fig:synthetic}c where $n= 3000/m$ and $d=10$. From the figures, we can conclude that the time complexities with respect to $n$ are $\O(n^2)$, but the time complexities with respect to $d$ and $m$ are generally not in the worst cases, i.e. $\O(d^3)$ and $\O(m^3)$. The reasons are analyzed as follows. Compared to the CPU time consumed on constructing the kernel, which scales with $\O((nm)^2)$, the time consumed on the matrix inverse is quite small. Moreover, when $nm$ is given, more task number only means the multitask-kernel matrix is more sparse, so that the methods need even less time to calculate the kernel matrix. This accounts for the interesting phenomenon of Fig. \ref{fig:synthetic}c.

\begin{figure}[t]
\scalebox{1}{
\centerline{\includegraphics[width=3.25in]{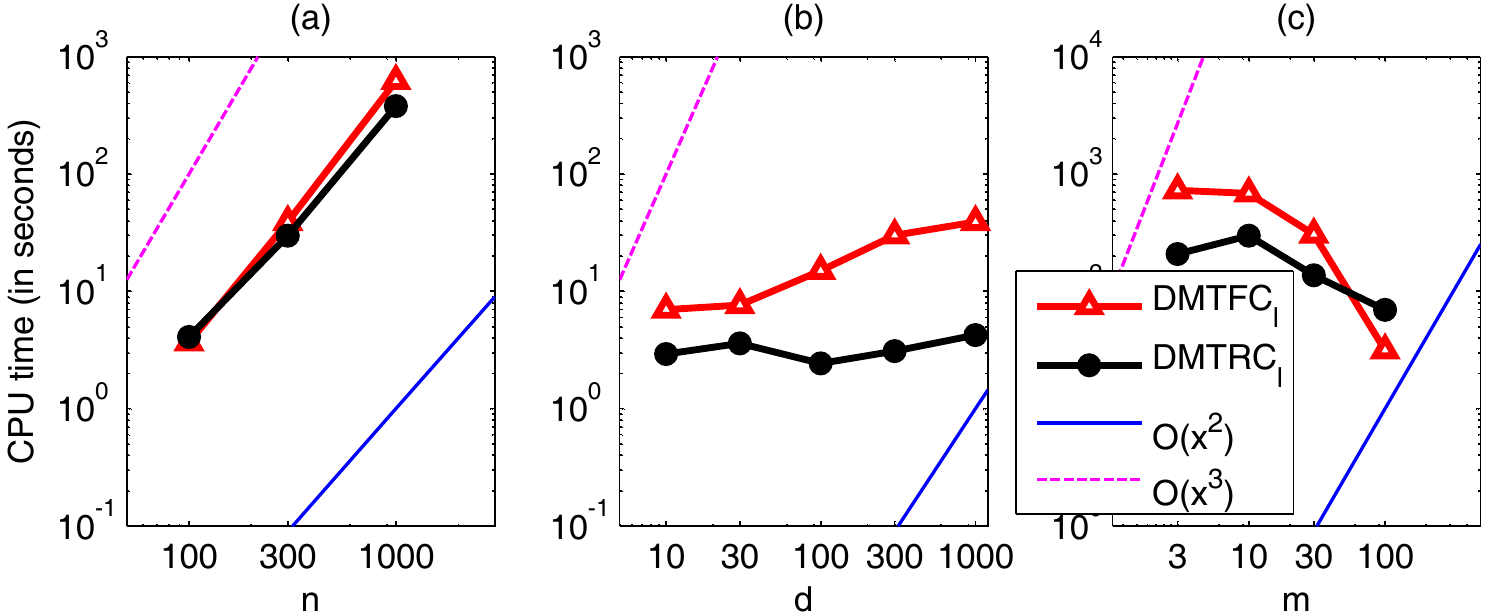}}}
\caption{{Time complexities with respect to the data set size of each task ($n$), feature dimension ($d$), and number of tasks ($m$). The symbol $x$ in the legends $\O(x^2)$ and $\O(x^2)$ stands for $n$, $d$ or $m$ in (a), (b) or (c) respectively.} }
\label{fig:synthetic}
\end{figure}

\subsection{Results on Multi-Domain Newsgroups Dataset}
The 20-newsgroups dataset is a widely used benchmark dataset that is a collection of about 20000 messages collected from 20 different \textit{usenet} newsgroups, 1000 messages from each. After postprocessing, each message is a vector with 26214 dimensions. We define a three class MTC job on the 20-newsgroups in Table \ref{tab:Compare_m3c}. From the table, we can see that Tasks 1 and 2 are highly related, Tasks 1 to 5 are somewhat related, while Task 6 seems an outlier task.
Based on the above task definition, we generate 4 MTC problems by randomly selecting 5\%, 10\%, 20\%, and 40\% of the data from each class, so as to observe how the data number influences the effectiveness of DMTC.  Because most algorithms are quite inefficient in high dimensional datasets, we use PCA to project the dataset to a 100-dimensional subspace. DMTC and DSTC only use the linear kernel. The DMTRC$_l$ and DSTC$_l$ without the PCA projection, which are denoted as *DMTRC$_l$ and *DSTC$_l$ respectively, will also be investigated.

\begin{table} [t]
\caption{\label{tab:Compare_m3c} {Task definition on the 20-newsgroups dataset.}}
\centerline{\scalebox{0.8}{
\begin{tabular}{l||l}
 \hline
ID & Names of the classes\\
\hline
\hline
Task 1 & comp.sys.mac.hardware \textit{vs.} rec.sport.hockey \textit{vs.} sci.electronics\\
\hline
Task 2 & comp.sys.ibm.pc.hardware \textit{vs.} rec.sport.baseball \textit{vs.} sci.crypt\\
\hline
Task 3& comp.windows.x \textit{vs.} rec.autos \textit{vs.} talk.politics.guns\\
\hline
Task 4 & comp.os.ms-windows.misc \textit{vs.} sci.med \textit{vs.} talk.politics.mideast \\
\hline
Task 5 & rec.motorcycles \textit{vs.} sci.space \textit{vs.} talk.politics.misc\\
\hline
Task 6 & misc.forsale \textit{vs.} alt.atheism \textit{vs.} soc.religion.christian\\
\hline
\end{tabular}}}
\end{table}

\begin{figure}[t]
\scalebox{1}{
\centerline{\includegraphics[width=3.25in]{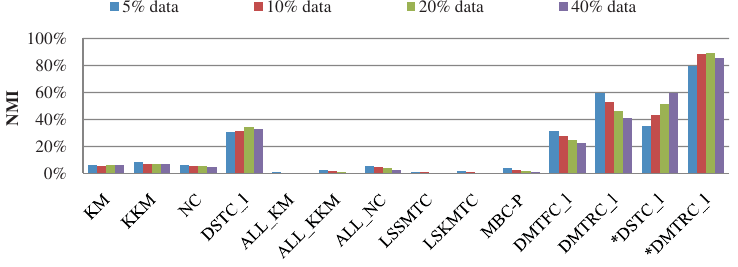}}}
\caption{{NMI comparison on the 20-newsgroups dataset. $a\%$ is short for ``experiments running with $a\%$ data of the dataset.''} }
\label{fig:20_acc}
\end{figure}

\begin{figure}[t]
\scalebox{1}{
\centerline{\includegraphics[width=3.25in]{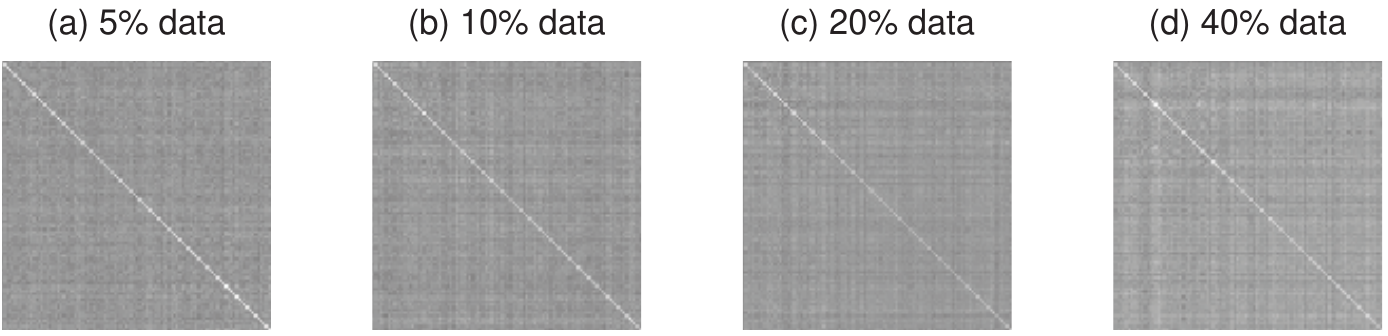}}}
\caption{{Visualizations of $\DD$ of DMTFC$_l$ on the 20-newsgroups dataset.} }
\label{fig:20_DD}
\end{figure}

\begin{figure}[t]
\scalebox{1}{
\centerline{\includegraphics[width=3.25in]{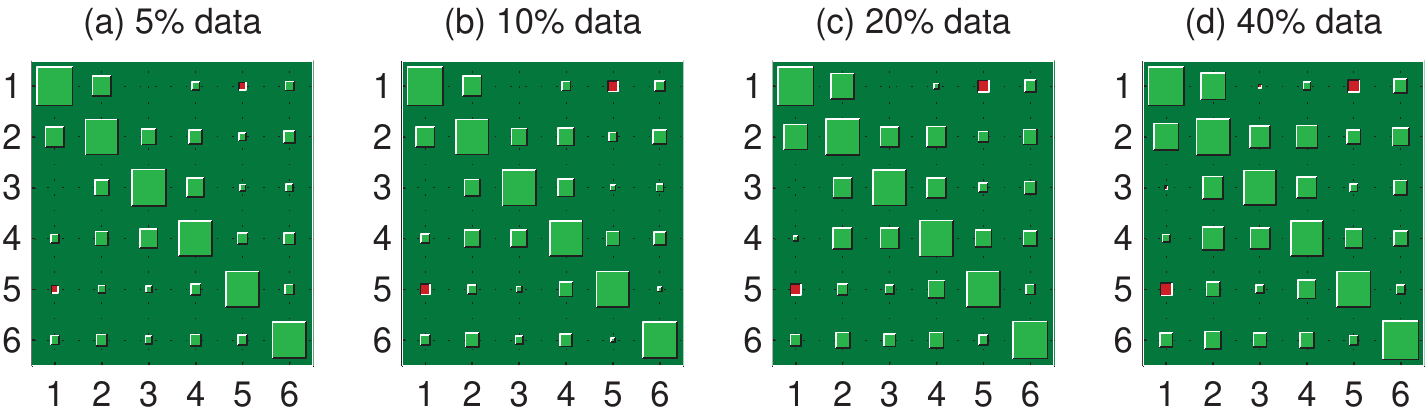}}}
\caption{{Hinton diagrams of $\bm\Omega$ of DMTRC$_l$ on the 20-newsgroups dataset.} }
\label{fig:20_NMI2}
\end{figure}

\begin{figure}[t]
\scalebox{1}{
\centerline{\includegraphics[width=3.25in]{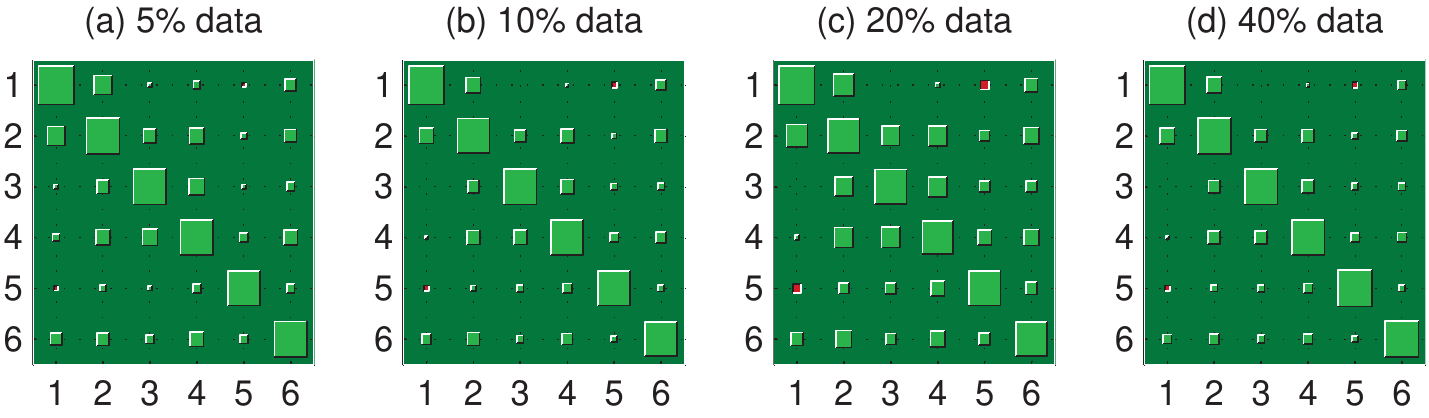}}}
\caption{{Hinton diagrams of $\bm\Omega$ of *DMTRC$_l$ on the 20-newsgroups dataset.} }
\label{fig:20_NMI}
\end{figure}

\begin{figure}[t]
\scalebox{1}{
\centerline{\includegraphics[width=3.25in]{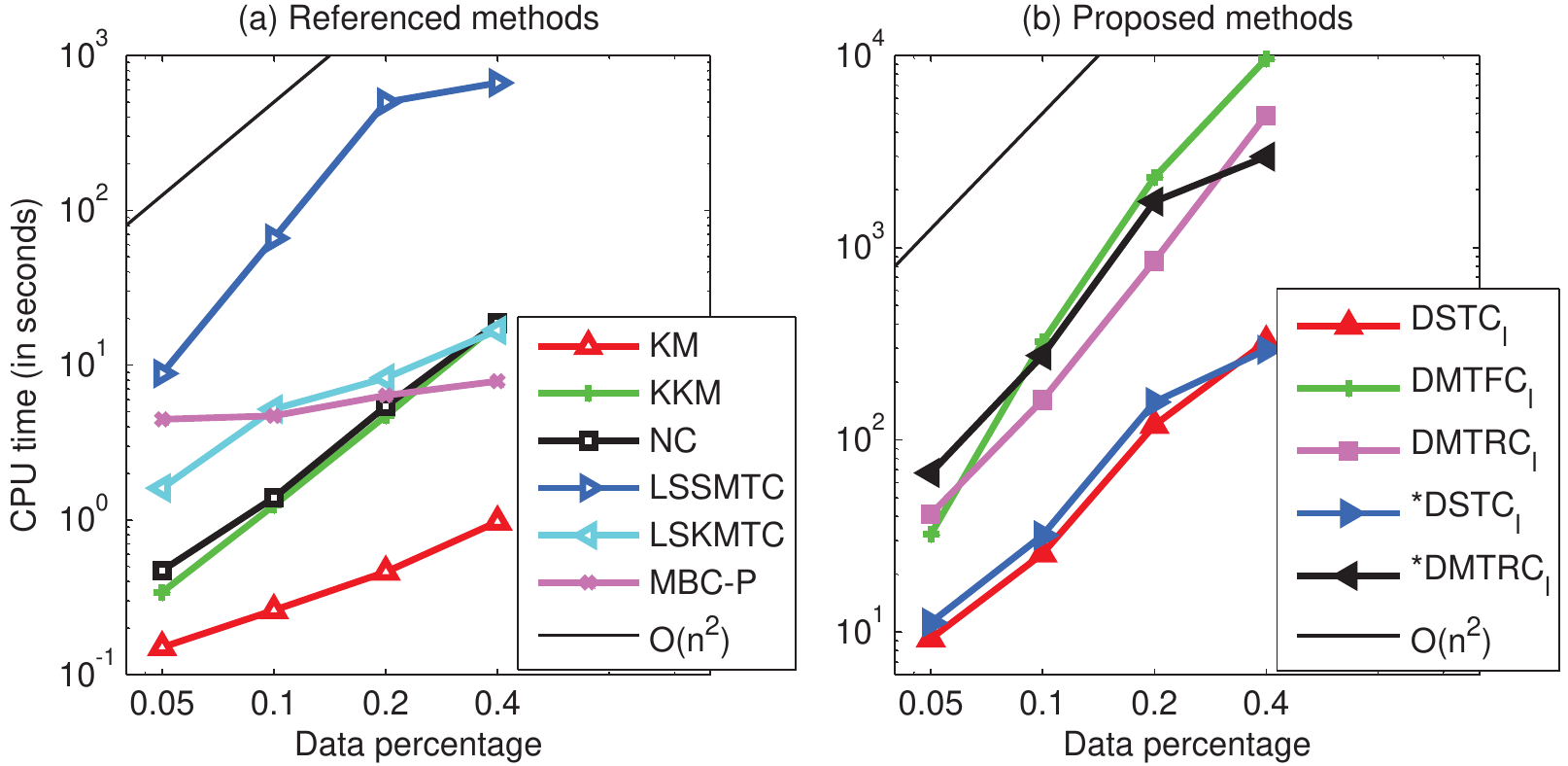}}}
\caption{{CPU time comparison on the 20-newsgroups.} }
\label{fig:20_CPU}
\end{figure}

Fig. \ref{fig:20_acc} shows the NMI comparison. From the figure, we can observe the following experimental phenomena. First, the proposed convex discriminative clustering algorithms are apparently better than the referenced methods in the same experimental environment. Second, DMTRC$_l$ is much better than DSTC$_l$ which shows that the task relationship is learned successfully. Third, DMTFC$_l$ is slightly worse than DSTC$_l$ which means that we cannot learn a strong shared feature representation across the tasks. This phenomenon might be caused by the PCA projection where much useful information for constructing the feature representation is lost, however, we cannot get its performance in the original dataset due to its inefficiency in high dimensional data.  Fourth, when the PCA projection is used to form the experimental environment, the performances of the clustering algorithms are getting worse when more data is used. On the contrary, when PCA is not used, the performances of both *DSTC$_l$ and *DMTRC$_l$ are getting better. This phenomenon tells us that when more data is available, the features should provide more abundant information so as to make the models available to be more complicated for describing the more variant distributions. It also shows the power of DSTC and DMTRC on high dimensional datasets. Moreover, it demonstrates that the power of the proposed discriminative clusterings do not rely on the predefined models for describing the data distribution which is an apparent superiority to the generative clusterings.

To show how well the feature representation is learned, we visualize $\DD$ of DMTFC$_l$ in Fig. \ref{fig:20_DD}. The figure shows that most features are considered as mutually independent, which might account for the ineffectiveness of DMTFC$_l$.

To demonstrate how well the task relationship is learned, we list the hinton diagrams of $\bm\Omega$ of DMTRC$_l$ and *DMTRC$_l$ in Figs. \ref{fig:20_NMI2} and \ref{fig:20_NMI} respectively. The figures show that both methods can learn the task relationships in different percentages of data equivalently well. They also show that the task relationship is different from what we have defined in Table \ref{tab:Compare_m3c}. As an example, Task 6 is originally designed as an outlier task, but it contributes to the performance positively. This phenomenon is worth of further study.

Fig. \ref{fig:20_CPU} gives the CPU time comparison. The figure shows that although the proposed methods have higher absolute time, both the proposed algorithms and the referenced methods have a time complexity of $\O(n^2)$ except KM, LSKMTC and MBC-P, which means that they are all unavailable for large-scale problems.

The results on each individual task and the stability analysis are described in the supplementary materials.

\subsection{Results on Multi-Domain Sentiment Dataset}

The multi-domain sentiment dataset is a widely used benchmark dataset that was originally designed for the MTL research propose. It contains product reviews taken from Amazon.com from many product types (domains or tasks). For a convenient comparison with the supervised MTFL and MTRL, we adopt the same experimental setting as \cite{zhang2010convex}. Specifically, the dataset in use is a postprocessed version that aims to classify the reviews of some products to two classes: positive or negative reviews. It
contains four binary-class tasks: books, DVDs, electronics, and kitchen appliances. Each task contains 2000 observations, in which 1000 reviews are labeled as positive and the other 1000 as negative. Each observation is a vector with 473853 dimensions.Note that we discarded 3 features that contain unrecognized characters.
 We generate 3 MTC problems by randomly selecting 10\%, 30\%, and 50\% of the data from each task. Other experimental settings are the same as those on the 20-newsgroups dataset.

 Fig. \ref{fig:S_acc} gives the NMI comparison. The experimental phenomena are quite similar with those on the 20-newsgroups dataset. The only difference is that when more data is available and when PCA is used to project the high dimensional dataset to a low dimensional space, the clustering algorithms are generally getting better on the sentiment dataset while the algorithms are getting worse on the 20-newsgroups dataset. This might be caused by the difficulties of the datasets. That is to say, projecting the data to 100 dimensional subspace is enough to catch the useful information on the sentiment dataset while doing so is not enough on the 20-newsgroups dataset. To support this explanation, we visualize $\DD$ of DMTFC$_l$ in Fig. \ref{fig:S_DD} and compare it with the visualizations of $\DD$ in Fig. \ref{fig:20_DD}. We can see that the filters $\DD$ on the sentiment set are more effective than those on the 20-newsgroups set.

 \begin{figure}[t]
\scalebox{1}{
\centerline{\includegraphics[width=3.25in]{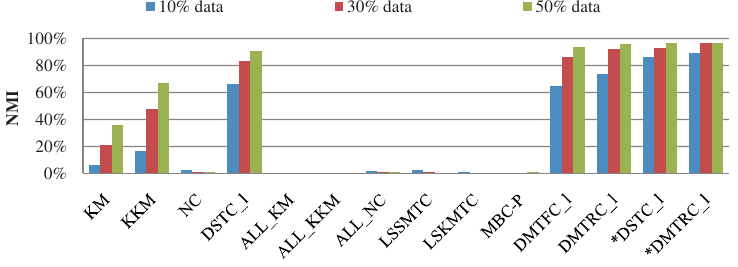}}}
\caption{{NMI comparison on the sentiment dataset. $a\%$ is short for ``experiments running with $a\%$ data.''} }
\label{fig:S_acc}
\end{figure}

\begin{figure}[t]
\scalebox{1}{
\centerline{\includegraphics[width=2.5in]{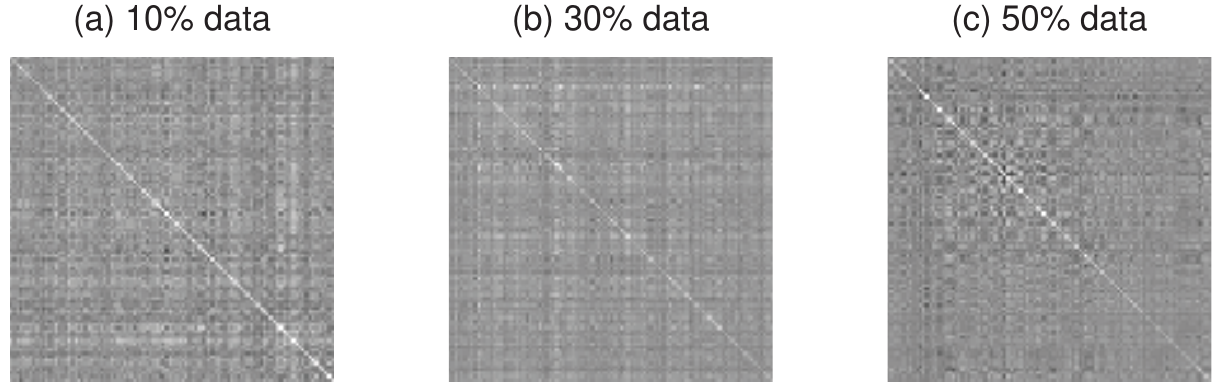}}}
\caption{{Visualizations of $\DD$ of DMTFC$_l$ on the sentiment dataset.} }
\label{fig:S_DD}
\end{figure}

\begin{figure}[t]
\scalebox{1}{
\centerline{\includegraphics[width=2.5in]{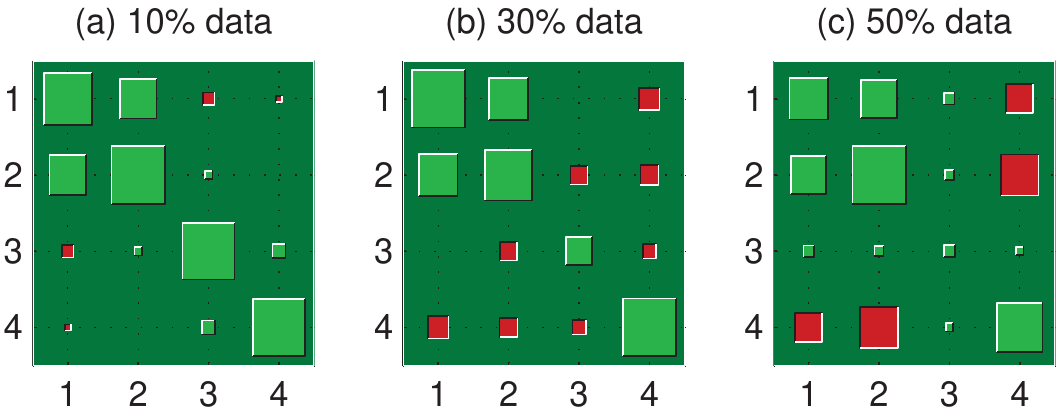}}}
\caption{{Hinton diagrams of $\bm\Omega$ of DMTRC$_l$ on the sentiment dataset.} }
\label{fig:S_NMI2}
\end{figure}

\begin{figure}[t]
\scalebox{1}{
\centerline{\includegraphics[width=2.5in]{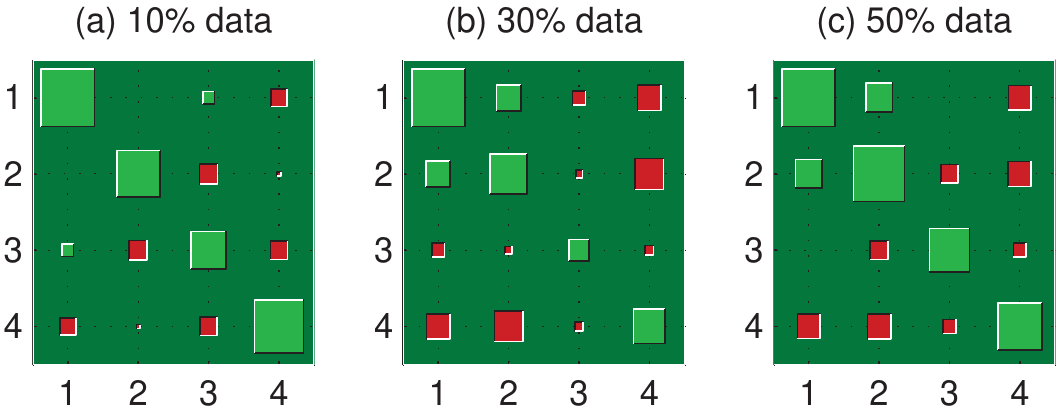}}}
\caption{{Hinton diagrams of $\bm\Omega$ of *DMTRC$_l$ on the sentiment dataset.} }
\label{fig:S_NMI}
\end{figure}

\begin{figure}[t]
 \centering
 \subfigure{
         \begin{minipage}[t]{2.75in}
         \includegraphics[width=\textwidth]{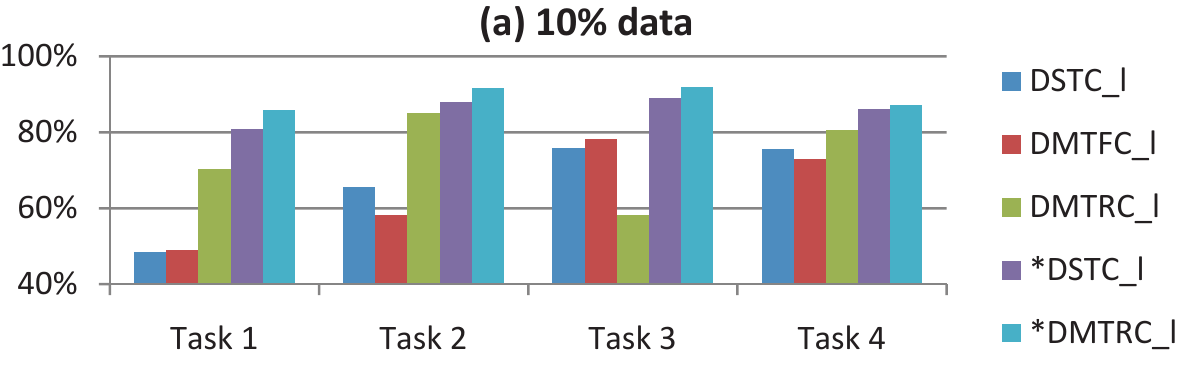}
         \end{minipage}
 }
 \subfigure{
         \begin{minipage}[t]{2.75in}
         \includegraphics[width=\textwidth]{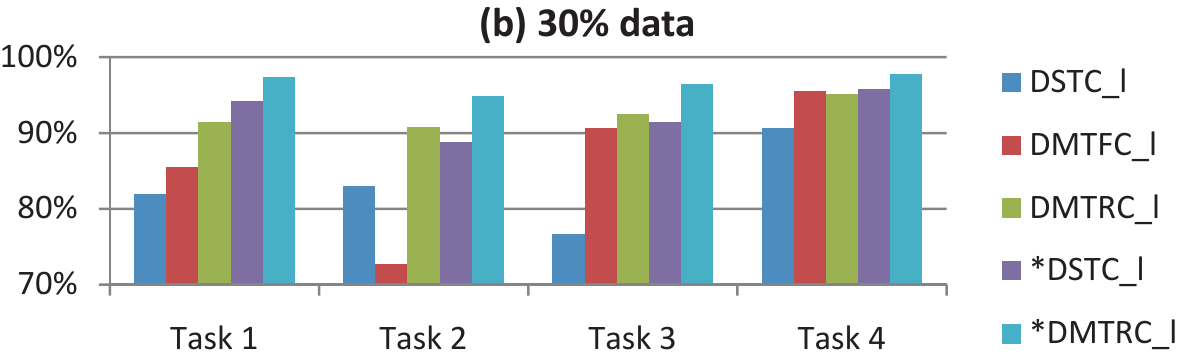}
         \end{minipage}
 }
 \subfigure{
         \begin{minipage}[t]{2.75in}
         \includegraphics[width=\textwidth]{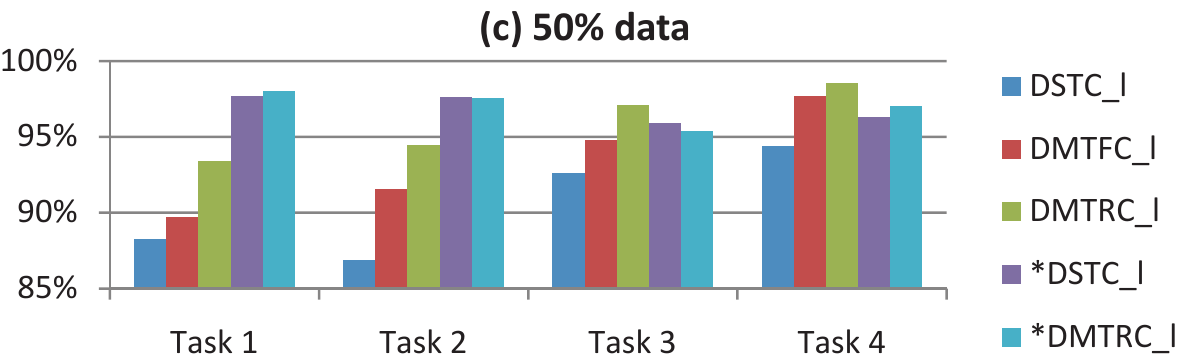}
         \end{minipage}
 }
 \caption{{NMI comparison between the proposed methods on the individual tasks with respect to different percentages of data on the sentiment dataset.} }
 \label{fig:indiv}
 \end{figure}

We provide the hinton diagrams of $\bm\Omega$ of DMTRC$_l$ and *DMTRC$_l$ in Figs. \ref{fig:S_NMI2} and Figs. \ref{fig:S_NMI}. We further provide the performance of the proposed algorithms on the individual tasks in Fig. \ref{fig:indiv}. The experimental phenomena in Fig. \ref{fig:indiv} are consistent with those in Fig. \ref{fig:S_acc} and are comparable with those yielded by the supervised counterparts of the proposed clusterings, i.e. MTFL and MTRL (see \cite[Section 4.3]{zhang2010convex}).
Finally, we list the running time of the methods in Fig. \ref{fig:3.xxx}. The results are consistent with the results in Fig. \ref{fig:20_CPU}.

\begin{figure}
\scalebox{1}{
\centerline{\includegraphics[width=3.25in]{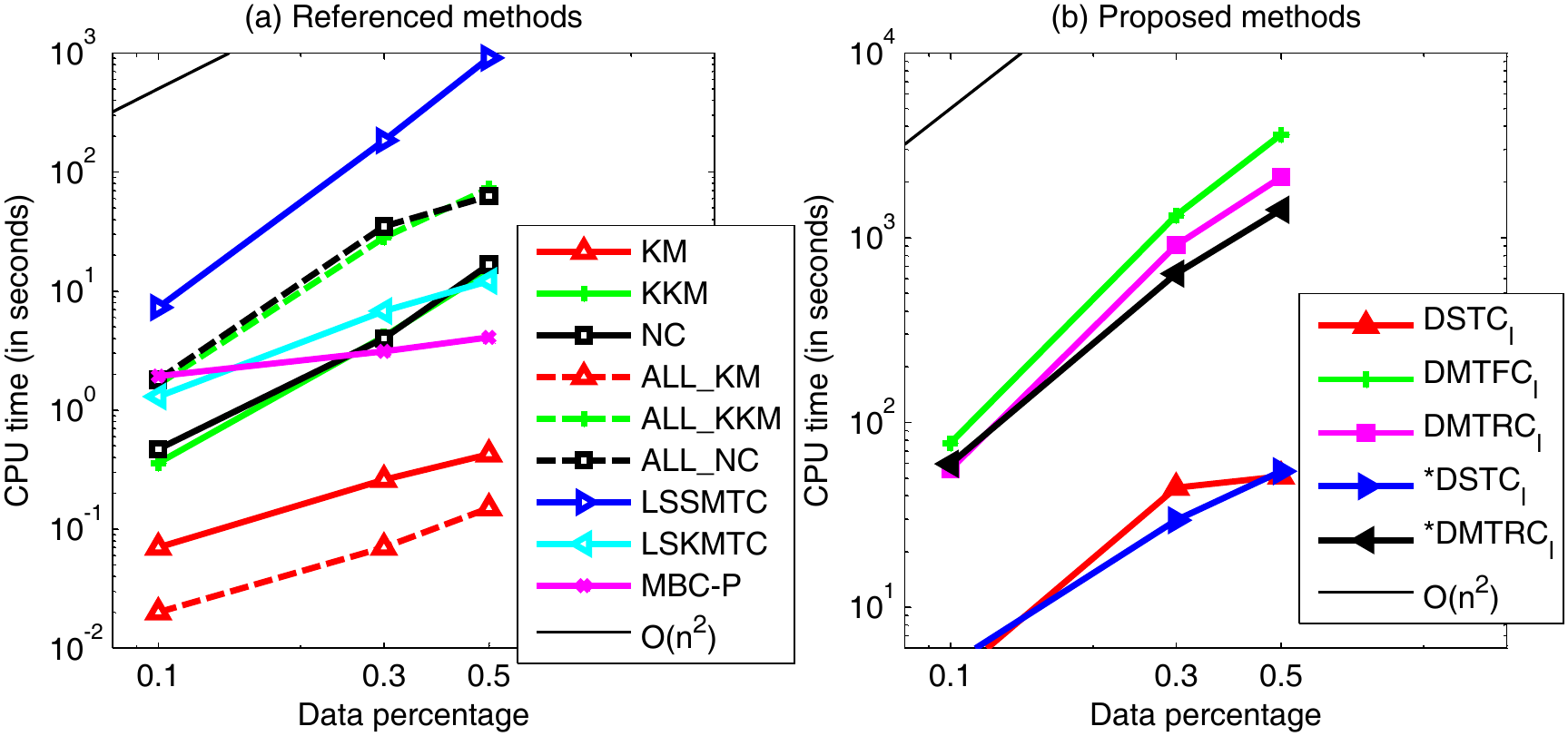}}}
\caption{{CPU time comparison on the sentiment dataset.} }
\label{fig:3.xxx}
\end{figure}

\section{Conclusions and Future Work}\label{sec:conclusion}
In this paper, we have proposed a novel Bayesian DMTC framework. Within the framework, we have implemented two multiclass DMTC objectives by specifying the framework with four assumptions. The first one, named DMTFC, works under the multivariate Gaussian prior that models a shared feature representation across tasks, while the second one, named DMTRC, models the task relationship. Both objectives are formulated as difficult MIP problems. We have further relaxed the MIP problems to convex optimization problems and solve the relaxed problems efficiently in a uniform alternating optimization procedure. Technically, the two convex DMTC algorithms can be seen as the objective combination of the supervised MTFL/MTRL and the unsupervised SVR-M3C. Experimental comparison with 7 STC algorithms as well as 3 state-of-the-art MTC algorithms on the pendigits, multi-domain newsgroups and multi-domain sentiment datasets demonstrated the effectiveness of the proposed algorithms.

\ifCLASSOPTIONcompsoc
\else

\fi



\ifCLASSOPTIONcaptionsoff
  \newpage
\fi

\bibliographystyle{IEEEtran}
\bibliography{zxlrefs} 

%

\end{document}